\documentclass[11pt]{article}

\usepackage[margin=1in]{geometry}

\usepackage{amsmath,amssymb,amsfonts,amsthm,mathtools}
\usepackage{bm}

\usepackage{graphicx}
\usepackage{enumitem}
\usepackage{float}
\usepackage{booktabs}
\usepackage{microtype}
\usepackage{caption}
\usepackage{subcaption}
\usepackage{xr}

\usepackage[authoryear, round]{natbib}
\bibpunct{(}{)}{; }{a}{}{, }
\usepackage{xcolor}
\usepackage[colorlinks=true,linkcolor=blue,citecolor=blue,urlcolor=blue]{hyperref}
\usepackage[nameinlink,noabbrev]{cleveref}

\newcommand{\bfx}{ \mathbf{x}}

\newcommand{\bfz}{ \mathbf{z}}
\newcommand{\bfD}{ \mathbf{D}}

\newcommand{\bbN}{ \mathbb{N}}
\newcommand{\bbP}{ \mathbb{P}} 
\newcommand{\bbR}{ \mathbb{R}}

\newcommand{\calC}{\mathcal{C}}

\newcommand{\calF}{\mathcal{F}}

\newcommand{\calN}{\mathcal{N}}

\newcommand{\calX}{ \mathcal{X}}

\newcommand{\balpha}{ \boldsymbol{\alpha}}
\newcommand{\bbeta}{ \boldsymbol{\beta}}

\newcommand{\bxi}{ \boldsymbol{\xi}}
\newcommand{\bzeta}{\boldsymbol{\zeta}}
\newcommand{\bJ}{\boldsymbol{J}}

\theoremstyle{plain}
\newtheorem{theorem}{Theorem}[section]
\newtheorem{lemma}[theorem]{Lemma}
\newtheorem{proposition}[theorem]{Proposition}
\newtheorem{corollary}[theorem]{Corollary}

\theoremstyle{definition}

\newtheorem{remark}[theorem]{Remark}

\title{Posterior Contraction of L\'evy Adaptive B-spline Regression in Besov Spaces}

\author{
  Jeunghun Oh\thanks{Department of Statistics, Seoul National University. Email: \texttt{nomad1994@snu.ac.kr}}
  , 
  Sewon Park\thanks{Department of Statistics, Sookmyung Women's University.}
  , and
  Jaeyong Lee\thanks{Department of Statistics, Seoul National University.}
}

\date{\today}

\begin{document}

\maketitle

\begin{abstract}
We investigate the asymptotic properties of the L\'evy Adaptive B-spline (LABS) regression model, a Bayesian nonparametric method that incorporates B-spline kernels into the L\'evy Adaptive Regression Kernel (LARK) model. LABS applies splines of varying degrees with independently defined knots, yielding a flexible model class capable of adapting to irregular and locally structured features of the true function. Within the nonparametric regression framework with univariate random design and Gaussian errors, we establish that the LABS posterior contracts around the true function in Besov classes at nearly minimax-optimal rates, up to a logarithmic factor, while adapting automatically to unknown smoothness. This study contributes to filling a gap in the literature, where theoretical results on posterior contraction of the LARK model in Besov spaces remain scarce. Simulation experiments on standard test functions in Besov spaces, including Blocks, Bumps, HeaviSine, and Doppler, complement the theoretical results and demonstrate the practical utility of LABS.
\end{abstract}

\section{Introduction}

Nonparametric function estimation is widely used in both theory and applications. Its goal is to recover the underlying function that generates the observed data, with applications in signal processing, medical image analysis, and anomaly detection. From a Bayesian perspective, a prior is placed on the regression function $f$, and inference is made through the induced posterior. Formally, Bayesian nonparametric regression is expressed as
\begin{align}  \label{eq:BayesModel}
\begin{split}
    Y^{(n)} \, | \, f \sim P_f,    \\
    f \sim \Pi,
\end{split}
\end{align}
where $P_f$ is the statistical model for the observed data $Y^{(n)}$, and $\Pi$ is a prior on functions $f:\mathcal{X} \rightarrow \mathbb{R}$. Examples include Gaussian process regression \citep{williams2006gaussian}, Bayesian additive regression trees \citep{chipman2010bart}, and Dirichlet process mixtures \citep{muller1996bayesian}. Bayesian models are attractive due to their conceptual simplicity, the probabilistic interpretation enabled by credible sets, and the regularisation effect of priors, which mitigates overfitting and provides robustness to outliers.

This paper focuses on a Bayesian random spline model, in particular a regression framework based on the L\'evy Adaptive Regression Kernel (LARK) model \citep{wolpert2011Stochastic} with B-splines as kernels, and studies its posterior contraction properties in Besov spaces. B-splines, owing to their recursive definition \citep{de1978practical}, are memory-efficient to compute. They consist of piecewise polynomials defined on a set of knots, each with local support. This local structure allows B-splines to capture abrupt changes in slope and curvature of the underlying function, enhancing interpretability. Thanks to their computational efficiency and local adaptability, B-spline models are widely used in computer graphics, engineering design, and biomedical applications such as EEG/ECG analysis. More recently, B-splines have been employed in deep learning through Kolmogorov–Arnold Networks (KANs) \citep{liu2025kan}, where they are proposed as an alternative to multilayer perceptrons. In KANs, each activation function in a feedforward neural network is represented as a sum of univariate B-spline series over input coordinates, which parallels the idea of additive models in statistics. With their flexibility in modelling activation functions and interpretability, KANs have gained attention across diverse fields.

In practical applications of spline models, it is natural to desire both an appropriate number of knots and their optimal placement. This is referred to as the knot selection problem. In the Bayesian literature, various approaches have been proposed to automate this process by assigning prior distributions to the knots. For instance, \citet{denison1998bayesian} and \citet{biller2000adaptive} placed discrete uniform priors on a predetermined set of candidate knots. \citet{dimatteo2001bayesian} and \citet{lindstrom2002bayesian} proposed using a Poisson prior on the number of knots and a continuous uniform prior on their locations in univariate B-spline models. However, these works did not further investigate the asymptotic properties of the resulting posterior distributions, such as posterior contraction.

Analyses of posterior contraction rates for Bayesian random spline regression models began to emerge after the development of general theories
\citeauthor{ghosal2000convergence} (\citeyear{ghosal2000convergence}, \citeyear{ghosal2007convergence})
\citet{huang2004convergence} studied the univariate random-design regression problem by proposing a sieve prior, introducing randomness into the number of knots and the degree of B-splines. The knots were simple knots with equal spacing, i.e., cardinal B-splines. They showed that the resulting model achieves minimax-optimal and adaptive rates in $L^2$ norm for functions in an $s$-Sobolev space. \citet{de2012adaptive} considered the multivariate fixed-design regression setting and proposed a Gaussian random spline prior based on Gaussian processes and tensor-product B-splines,  with fixed degree. By imposing suitable priors on the number of basis functions (equivalently, knots) through conditionally Gaussian priors, they obtained rate-adaptive result. The model achieves minimax-adaptive posterior contraction rates in empirical $L^2$ distance for functions in an $s$-Hölder space. \citet{shen2015adaptive} investigated both univariate and multivariate fixed-design regression, considering wavelets, splines, and other random basis expansions, and established sufficient conditions for Bayesian adaptation. As in earlier works, they employed cardinal B-splines, showing that tensor-product spline bases achieve minimax-adaptive posterior contraction rates in $s$-Hölder spaces under Hellinger distance. Finally, \citet{belitser2014adaptive} demonstrated that univariate random B-spline models with non-uniform knots also attain minimax-adaptive posterior contraction rates in $s$-Hölder spaces under Hellinger distance. Unlike previous works, this approach uses non-uniform knots, offering greater flexibility than models restricted to equally spaced knots, while still fixing the spline degree.

In summary, existing studies on random spline models and their posterior contraction can be broadly classified into two directions. First, most analyses focus on function spaces such as Hölder and Sobolev classes. Second, random spline models typically employ B-splines with fixed degree and uniformly spaced knots.

This work aims to extend the two directions identified above. First, we introduce the L\'evy Adaptive B-spline Regression (LABS) \citep{park2023levy}, a model class that is more flexible than those in previous random spline studies. Second, we establish that LABS achieves adaptive posterior contraction rates that are nearly minimax-optimal, up to a logarithmic factor, for regression functions in Besov spaces—going beyond the Hölder and Sobolev spaces typically considered in the literature. The contributions of this paper are as follows:

\begin{enumerate}
    \item We introduce LABS, a more general and flexible model than existing random spline approaches. LABS employs B-spline bases within the L\'evy Adaptive Regression Kernel (LARK) model \citep{wolpert2011Stochastic}. Traditional spline models define bases across a shared set of knots, with neighbouring bases overlapping at certain knots. In contrast, under the LARK model, each basis function is defined separately by its own set of knots, and splines of varying degrees can coexist within the model. Consequently, the model space considered by LABS is strictly larger than those in prior work. Moreover, allowing splines of different degrees enables the posterior to better accommodate irregular local features of the true function, without relying on a fixed global spline degree.
    
    \item We establish asymptotic properties of the LABS posterior for functions in Besov spaces. Unlike Hölder or Sobolev spaces—commonly used in previous studies—which assume the existence of derivatives, Besov spaces are defined through finite differences, making them a more general framework. We show that the LABS posterior contracts at rates that are nearly minimax-optimal up to a logarithmic factor. Furthermore, these rates are adaptive: by assigning super-polynomial rates in $n$ to the hyperparameters of the prior, one can achieve near-optimal adaptive contraction.
    
    Along with this, we incorporate the problem of estimating the unknown error variance into our theoretical framework. From a Bayesian perspective, it is unnatural to separate the estimation of the underlying  function and variance. Accordingly, our approach treats both quantities as unknown parameters and establishes joint posterior contraction rates for the function and the variance.
    
    \item We provide empirical results supporting the theoretical findings. Under suitable hyperparameter conditions that guarantee posterior contraction, LABS shows competitive performance relative to alternative nonparametric methods in estimating functions with irregular patterns. We illustrate its performance through simulation studies using standard test functions, including Blocks, Bumps, HeaviSine, and Doppler.
\end{enumerate}

The remainder of the paper is organised as follows. Section~\ref{sec:preliminaries} presents preliminaries, including notation, statistical model assumptions, and a brief introduction to Besov spaces. Section~\ref{sec:levy} introduces the modelling framework: we first describe the LARK model and then define LABS based on it, along with the hyperparameter conditions required for posterior contraction in Besov spaces. Section~\ref{sec:results} provides the main theoretical results on posterior contraction rates of LABS in Besov spaces. Section~\ref{sec:simulation} reports empirical results supporting the theory, including simulations under various settings with standard test functions in the Besov space literature, such as Blocks, Bumps, HeaviSine, and Doppler. Section~\ref{sec:concluding} concludes with a summary and discussion of the main contributions.

\section{Preliminaries}
\label{sec:preliminaries}

\subsection{Notation}
Let $\mathbb{N}$ and $\mathbb{R}$ denote the sets of natural and real numbers, respectively, and set $\mathbb{N}_0 := \mathbb{N} \cup \{0\}$. For sequences $a_n, b_n > 0$, write $a_n \lesssim b_n$ if there exists $C > 0$ such that $a_n \le C b_n$. For $a, b \in \mathbb{R}$, let $a \wedge b := \min(a,b)$ and $a \vee b := \max(a,b)$. The absolute value of $x \in \mathbb{R}$ is denoted by $|x|$, and for a set $S$, $|S|$ denotes its cardinality. The floor and ceiling functions are written as $\lfloor x \rfloor$ and $\lceil x \rceil$, respectively. For $\bfx = (x_1,\ldots,x_d) \in \mathbb{R}^d$ with $d \ge 2$, define $\|\bfx\|_p := (\textstyle\sum_{i=1}^d |x_i|^p)^{1/p}$ for $0 < p < \infty$, and $\|\bfx\|_\infty := \max_{i=1,\ldots,d} |x_i|$. For a measurable function $f:\mathcal{X} \to \mathbb{R}$ defined on
$\mathcal{X} \subset \mathbb{R}$ and a positive measure $\mu$, define the $L^p(\mu)$-norm by $\|f\|_{L^p(\mu)} := \left( \int_{\mathcal{X}} |f(x)|^p \, d\mu(x) \right)^{1/p}$ for $0 < p < \infty$. When $\mu$ is the Lebesgue measure on $[0,1]^d$ with $d \ge 1$, we write $\|f\|_{L^p}$ for simplicity. For a Lebesgue measure defined on $\mathcal{X} \subset \mathbb{R}$, we explicitly write $\|f\|_{L^p(\mathcal{X})}$. For $p = \infty$, we define $\|f\|_{L^\infty} := \sup_{x \in [0,1]} |f(x)|$ and $\|f\|_{L^\infty(\mathcal{X})} := \sup_{x \in \mathcal{X}} |f(x)|$. The $\epsilon$-covering number of a set $T$ under a metric $d$ is denoted by $\mathcal{N}(\epsilon, T, d)$. Let $K(p, q) := \int p  \log (p / q)  d\nu$ be the Kullback--Leibler divergence between two densities $p$ and $q$ with respect to $\nu$, and let $V_{2, 0} (p, q) := \int p | \log (p/q) - K(p,q) |^2 d\nu$ denote the second-order centered Kullback--Leibler variation.

For a random variable $X$ following a Gamma distribution with parameters $\alpha, \beta > 0$, we write $X \sim \text{Gam}(\alpha, \beta)$, and $X \sim \text{Inv\text{-}Gam}(\alpha, \beta)$ when $X$ follows an inverse-Gamma distribution. A normal distribution with mean $\mu \in \mathbb{R}$ and variance $\sigma^2 > 0$ is denoted by $X \sim N(\mu, \sigma^2)$, and a Poisson distribution with mean $\mu > 0$ by $X \sim \text{Poi}(\mu)$. A uniform distribution on a set $\mathcal{X}$ is written as $X \sim U(\mathcal{X})$. We write $X \sim \text{PP}(\nu)$ if $X$ is a Poisson point process with intensity measure $\nu$, and $X \sim \text{L\'evy}(\nu)$ if $X$ is a L\'evy process with L\'evy measure $\nu$.

\subsection{Model Assumptions}
The statistical model considered in this paper is a one-dimensional nonparametric regression model with Gaussian errors and random design points. From a Bayesian perspective, it is unnatural to treat the noise variance as known or to estimate it separately from the regression function. Accordingly, we incorporate the problem of estimating the unknown variance directly into our modelling and theoretical framework, and conduct posterior inference jointly on the function and the variance.

The data $\bfD^{(n)} = (x_i, y_i)_{i=1}^n$ are assumed to be observed from the following distribution:
\begin{align} \label{eq:statModel}
\begin{split}
    y_i \mid x_i &\stackrel{iid}{\sim} N(f_0(x_i), \sigma_0^2), \\
    x_i &\stackrel{iid}{\sim} P_X, \quad i = 1, \ldots, n,
\end{split}
\end{align}
where the true parameter of interest is $\theta_0 = (f_0, \sigma_0)$, consisting of an unknown real-valued function $f_0:[0,1] \to \mathbb{R}$ and an unknown variance $\sigma_0^2 > 0$. Throughout the paper, we consider the following assumptions on model \eqref{eq:statModel}:
\begin{enumerate}[label=(A\arabic*), ref=A\arabic*]
    \item\label{eq:RandDesigAssump} The distribution $P_X$ on the support $[0,1] \subset \mathbb{R}$ admits a bounded density $p_X$ with respect to the Lebesgue measure; that is, there exists $R > 0$ such that $\|p_X\|_\infty \le R$.
    \item\label{eq:VarAssump} There exist constants $0<\underline{\sigma}\le \overline{\sigma}<\infty$ such that
    $\underline{\sigma}\le \sigma_0 \le \overline{\sigma}$.
    \item\label{eq:UnifAssump} The true regression function is uniformly bounded:
    $\|f_0\|_{L^\infty}\le F$ for some $F>0$.
\end{enumerate}

\subsection{Besov Space}
In our study, the true regression function $f_0$ is assumed to belong to the Besov space $B^s_{p, q} ([0, 1])$. There are several equivalent ways to construct Besov spaces. One common approach defines the Besov norm through the coefficients of an orthonormal wavelet basis of $L^2$, while another defines it directly via moduli of smoothness. We present the latter approach; see \citet{gine2021mathematical} for a detailed discussion.

The definition of the Besov norm begins with the notion of finite differences, which captures the uniform smoothness of a function $f$. The finite differences are then averaged under the $L^p$-norm to quantify the local smoothness of $f$. This quantity is called the modulus of smoothness. Averaging the modulus of smoothness again in the $L^q$-norm yields the Besov seminorm. The Besov space then consists of functions for which both the $L^p$-norm and this seminorm are finite.

For $f \in L^p$ with $0 < p \le \infty$ and $r \in \bbN$, the $r$-th finite difference with step size $h > 0$, is defined as
\begin{equation*}
    \Delta^r_h (f)(x) := 
    \begin{cases}
        \sum_{k=0}^r \binom{r}{k} (-1)^{r - k} f(x + kh), & x \in [0, 1 - rh], \\
        0, & \text{o.w.}
    \end{cases}
\end{equation*}
Note that finite differences, rather than derivatives, are used here. The $r$-th modulus of smoothness of $f$ is defined by
\begin{equation} \label{eq:r-thModSm}
    w_{r,p}(f, t) := \sup_{0 < h \le t} \| \Delta^r_h f \|_{L^p}, \quad t > 0.
\end{equation}
The quantity $w_{r,p}(f,t)$ in \eqref{eq:r-thModSm} represents the largest $L^p$-norm among all $r$-th finite differences $\Delta^r_h(f)$ with a step size $h$ up to radius $t > 0$, thus reflecting the local features of $f$. The averaging criterion for this local smoothness is the $L^p$-norm. The Besov seminorm is then defined as
\begin{equation*}
    | f |_{B^s_{p,q}} := 
    \begin{cases}
        \left( \displaystyle\int_0^\infty [ t^{-s} w_{r,p}(f,t) ]^q \frac{dt}{t} \right)^{1/q}, & q < \infty, \\
        \displaystyle\sup_{t > 0} \{ t^{-s} w_{r,p}(f,t) \}, & q = \infty.
    \end{cases}
\end{equation*}
That is, the Besov seminorm averages the modulus of smoothness over scales $t > 0$ using the $L^q$-norm as a measure of aggregation. Finally, for $0 < p, q \le \infty$, $s > 0$, $r = \lfloor s \rfloor + 1$, and $f \in L^p$, the Besov norm is defined as
\begin{equation*}
    \| f \|_{B^s_{p,q}} := \| f \|_{L^p} + | f |_{B^s_{p,q}}.
\end{equation*}
Hence, the Besov space $B^s_{p,q}[0,1]$ is the collection of functions with finite Besov norm:
\begin{equation*}
    B^s_{p,q}([0,1]) := \{ f : \| f \|_{B^s_{p,q}} < \infty \}.
\end{equation*}
The characteristics of the space depend on the parameters $s, p, q$. In particular, $s$ and $p$ primarily determine the smoothness and integrability of the functions, whereas $q$ plays a secondary role.

\section{L\'evy adaptive regression kernel}
\label{sec:levy}

This section introduces the \textit{L\'evy Adaptive B-spline regression} (LABS) model, the object of our posterior contraction analysis. The model specialises the kernel of the \textit{L\'evy Adaptive Regression Kernel} (LARK) to B-spline bases of multiple degrees. We first give a brief overview of LARK and then present the construction of LABS.

LARK represents the regression function $f$ via a stochastic kernel expansion. Consider the nonparametric Gaussian regression model,
\begin{equation} \label{eq:GaussMeanModel}
    y_i = f(x_i) + \varepsilon_i, \quad \varepsilon_i \sim N(0, \sigma^2), \quad i=1,\ldots,n.
\end{equation}
As a prior for $f:\mathcal{X}\to\mathbb{R}$, we consider a stochastic expansion in a pre-specified family of kernel functions. Let $\{(\beta_j,\omega_j)\}_{j<J}\subset\mathbb{R}\times\Omega$ be a random collection with random cardinality $J\le \infty$, and define a signed random measure $\mathcal{L}(d\omega)= \textstyle\sum_{0\le j<J}\beta_j \delta_{\omega_j}(d\omega)$. Then
\begin{equation}    \label{eq:LARKmean}
    f(x) := \int_\Omega g(x, \omega)\, \mathcal{L}(d\omega) \;=\; \sum_{0 \le j < J} \beta_j\, g(x, \omega_j), \quad x \in \mathcal{X},
\end{equation}
where $g:\mathcal{X}\times\Omega\to\mathbb{R}$ is a Borel-measurable generating (kernel) function and $\Omega$ is the parameter space of $g$, assumed to be a separable complete metric space.

\citet{wolpert2011Stochastic} construct a L\'evy random measure $\mathcal{L}(d\omega)$ as a prior for \eqref{eq:LARKmean} through a Poisson construction. This approach (i) is simple to construct and (ii) feasible for MCMC computation of the posterior. Let $\nu(d\beta, d\omega)$ be a L\'evy measure on $\mathbb{R}\times\Omega$ satisfying the local $L^1$ integrability condition
\begin{equation} \label{eq:L1int}
    \int \int_{\mathbb{R} \times K} (1 \wedge |\beta|)\, \nu(d\beta , d\omega) \;<\; \infty
    \quad \text{for all compact } K \subset \Omega.
\end{equation}
Define a Poisson random measure $N(d\beta , d\omega)\sim\mathrm{PP}(\nu)$. Then the L\'evy random measure $\mathcal{L}$ is given by the stochastic integral
\begin{equation}    \label{eq:LevyRandMeas}
    \mathcal{L}(A) := \int \int_{\mathbb{R} \times A} \beta\, N(d\beta , d\omega)
    \;=\; \sum_{0 \le j < J} \beta_j \, I_A (\omega_j),
\end{equation}
defined for any Borel set $A\subset\Omega$ with compact closure, where $J:=N(\mathbb{R}\times A)$. The set $\{ (\beta_j, \omega_j) \}_{0 \le j < J}$ is a Poisson random subset (PRS) of $\mathbb{R}\times A$. If the L\'evy measure is finite, $\nu(\mathbb{R}\times\Omega)=M<\infty$, then $J \sim \mathrm{Poi}(M)$ and $(\beta_j,\omega_j)\stackrel{iid}{\sim}\pi(d\beta , d\omega)$ for $j=1,\ldots,J$, where $\pi(d\beta,d\omega):=M^{-1}\nu(d\beta,d\omega)$. Condition \eqref{eq:L1int} ensures that the random variable $\mathcal{L}(A)$ in \eqref{eq:LevyRandMeas} is well-defined; moreover, for disjoint Borel sets $A_i\subset\Omega$, the random variables $\mathcal{L}(A_i)$ are independent and infinitely divisible. When $\nu (\mathbb{R} \times \Omega) = \infty$, we may have $J=\infty$, in which case $\mathcal{L}$ is defined only on countably many points $\omega_j\in A$ satisfying $\textstyle\sum_{\omega_j\in A}|\beta_j|<\infty$ a.s.

Defining the mean function by the stochastic integral of the kernel with respect to $\mathcal{L}$ yields a L\'evy random field $\mathcal{L}[g(x)]:=\textstyle\int_\Omega g(x,\omega)\,\mathcal{L}(d\omega)$. For any $g$ satisfying
\begin{equation}    \label{eq:L1intRandField}
    \int \int_{[-1, 1] \times \Omega} | \beta \, g(x, \omega) |\, \nu(d\beta ,  d\omega) \;<\; \infty,
\end{equation}
we have
\begin{equation*}
    f(x) := \mathcal{L}[g(x)] \;=\; \int_\Omega \int_\mathbb{R} \beta \, g(x, \omega)\, N(d\beta , d\omega)
    \;=\; \sum_{0 \le j < J} \beta_j \, g(x, \omega_j),
\end{equation*}
which is well-defined. Since \eqref{eq:L1int} holds, any bounded measurable $g$ with compact support satisfies \eqref{eq:L1intRandField}.

If the L\'evy measure fails to satisfy \eqref{eq:L1int}, both the L\'evy random measure and the induced random field may be ill-defined. In such cases, one can replace \eqref{eq:L1int} with weaker local $L^2$ integrability conditions. This extension is beyond the scope of the present paper; see \citet{wolpert2011Stochastic}, Section 2, Theorem 1 for details.

In summary, under the Gaussian error assumption \eqref{eq:GaussMeanModel}, the basic LARK specification is
\begin{align*}
\begin{split}
    \mathbb{E}[y \mid \mathcal{L},\theta] &:= f(x) = \int_\Omega g(x,\omega)\, \mathcal{L}(d\omega), \\
    \mathcal{L}\mid \theta &\sim \mathrm{L\acute{e}vy}(\nu), \\
    \theta &\sim \pi_\theta(d\theta),
\end{split}
\end{align*}
where $\theta$ denotes the hyperparameters of the L\'evy process.

Examples of L\'evy random fields include compound Poisson processes, gamma/symmetric-gamma random fields, and symmetric $\alpha$-stable random fields ($0<\alpha<2$). With the exception of the compound Poisson case, many have $\nu(\mathbb{R}\times\Omega)=\infty$, in which case the coefficient sequence $(\beta_j:\omega_j\in A)$ need not be $\ell^1$ and an approximation of $f$ is required. This motivates the use of a truncated prior based on the truncated L\'evy measure
\begin{equation*}
    \nu_\varepsilon (d\beta \, d\omega) := \nu (d\beta , d\omega)\, I(|\beta| > \varepsilon) \quad \text{for }\varepsilon > 0. 
\end{equation*}

When $\nu$ and $g$ satisfy the requisite $L^1$ or $L^2$ integrability, it is known that the model
\begin{align*}
\begin{split}
    f_\varepsilon(x) &:= \sum_{0\le j<J_\varepsilon}\beta_j\, g(x,\omega_j), \\
    J_\varepsilon &\sim \mathrm{Poi}(\nu_\varepsilon^+), \\
    (\beta_j,\omega_j)_{j=1,\ldots,J_\varepsilon}\mid J_\varepsilon &\sim \nu_\varepsilon(d\beta , d\omega)/\nu_\varepsilon^+,
\end{split}
\end{align*}
converges in probability to $f(x)$ as $\varepsilon\to 0$ \citep[Section~3, Corollary~2]{wolpert2011Stochastic}. Hence, in practice we employ the finite representation
\begin{align*}
\begin{split}
    y_i \mid f &\stackrel{ind}{\sim} N\!\left(f(x_i),\sigma^2\right), \\
    f(x_i) &:= \sum_{0\le j<J}\beta_j\, g(x_i,\omega_j), \\
    (\beta_j,\omega_j)_{0\le j<J}\mid J,\theta &\stackrel{iid}{\sim} \pi_\beta(\beta_j)\, d\beta_j\, \pi_\omega(d\omega_j), \\
    J\mid \theta &\sim \mathrm{Poi}(\nu_\varepsilon^+), \quad \nu^+_\varepsilon = \nu_\varepsilon(\mathbb{R} \times \Omega), \\
    \theta &\sim \pi_\theta(d\theta).
\end{split}
\end{align*}

We now turn to LABS. In \citet{park2023levy}, LABS employs B-splines of various degrees as the generating functions $g$. A B-spline of degree $k \in \mathbb{N}_0$ is uniquely determined by a nondecreasing sequence of $k+2$ knots $\boldsymbol{\xi}_k=(\xi_{k,1},\ldots,\xi_{k,k+2})$, and a basis function $B_k(\cdot;\boldsymbol{\xi}_k):\mathcal{X}\to\mathbb{R}$ is defined recursively \citep{de1978practical} by
\begin{align*}
    \begin{split}   \label{eq:BsplineDef}
        B_0(x; \boldsymbol{\xi}_0)  &:= \begin{cases}
            1,   & \text{if } \xi_{0,1} \leq x < \xi_{0,2}, \\
            0,   & \text{otherwise,}
        \end{cases} \\
        B_k(x; \boldsymbol{\xi}_k) &:= \frac{x - \xi_{k,1}}{\xi_{k, k + 1} - \xi_{k,1}} B_{k - 1} (x; \boldsymbol{\xi}^{L}_k)
        + \frac{\xi_{k, k + 2} - x}{\xi_{k, k + 2} - \xi_{k, 2}} B_{k - 1}(x; \boldsymbol{\xi}^{R}_k),
    \end{split}
\end{align*}
where $\boldsymbol{\xi}^L_k := (\xi_{k,1}, \xi_{k,2}, \ldots, \xi_{k,k + 1})$ and $\boldsymbol{\xi}^R_k := (\xi_{k,2}, \xi_{k,3}, \ldots, \xi_{k,k+2})$. B-splines (i) are piecewise polynomials of degree $k$ on each interval $[\xi_{k,i},\xi_{k,i+1})$, (ii) are $(k-1)$ times continuously differentiable at the knots, and (iii) admit memory-efficient computation due to the recursive structure.

LABS corresponds to a compound Poisson specification. The mean function, following the LARK representation \eqref{eq:LARKmean}, is written as
\begin{equation}    \label{eq:LABSmean}
    f(x) = \sum_{k\in S}\sum_{1\le l\le J_k} \beta_{k,l}\, B_k(x; \boldsymbol{\xi}_{k,l})
         = \sum_{k\in S}\int_{\Omega} B_k(x;\boldsymbol{\xi}_{k})\, \mathcal{L}_k(d\boldsymbol{\xi}_k),
\end{equation}
where $S\subset\mathbb{N}_0$ is the set of spline degrees used, $B_k$ serves as the generating function $g$, and $\Omega:=\mathcal{X}^{k+2}$. Each $J_k$ is independent with $J_k\sim \mathrm{Poi}(M_k)$, and $(\beta_{k,l},\boldsymbol{\xi}_{k,l})\stackrel{iid}{\sim}\pi_k(d\beta_k , d\boldsymbol{\xi}_k)$ for $l=1,\ldots,J_k$. The L\'evy random measures determining the stochastic integrals are
\begin{equation*}
    \mathcal{L}_k \stackrel{ind}{\sim} \mathrm{L\acute{e}vy}\!\left(\nu_k (d\beta_k , d\boldsymbol{\xi}_k)\right) \quad \text{for all } k \in S,
\end{equation*}
with $\nu_k(\mathbb{R}\times\Omega)=M_k<\infty$. Moreover, for each $k\in S$,
\begin{equation*}
    \int \int_{\mathbb{R} \times \Omega} \big( 1 \wedge | \beta_k \, B_k(x ; \boldsymbol{\xi}_k) | \big)\, \nu_k(d\beta_k , d\boldsymbol{\xi}_k) \;<\; \infty
    \quad \forall\, x \in \mathcal{X},
\end{equation*}
so that condition \eqref{eq:L1intRandField} is satisfied.

A hierarchical specification of LABS is suggested in \citet{park2023levy} as follows:
\begin{align}   \label{eq:LABSprior}
\begin{split}
    y_i \mid x_i &\stackrel{ind}{\sim} N\!\big(f(x_i), \sigma^2\big), \quad i = 1, \ldots, n,  \\
    f(x) &= \sum_{k \in S} \sum_{l = 1}^{J_k} \beta_{k,l}\, B_k (x ; \boldsymbol{\xi}_{k,l}),  \\
    \sigma^2 &\sim \mathrm{Inv\text{-}Gam} \!\left( \frac{r}{2}, \frac{r R}{2} \right),   \\
    J_k &\sim \mathrm{Poi}(M_k), \\
    M_k &\sim \mathrm{Gam} (a_k, b_k),  \\
    \beta_{k,l} &\stackrel{iid}{\sim} N (0, \phi_k^2), \quad l = 1, \ldots, J_k,    \\
    \boldsymbol{\xi}_{k,l} &\stackrel{iid}{\sim} U(\mathcal{X}^{k + 2}), \quad l = 1, \ldots, J_k.
\end{split}
\end{align}
Here, $\mathcal{X}$ is an interval containing the domain of the true function $f_0$; we take $[0,1]\subset\mathcal{X}$ to improve boundary behaviour.

In summary, $f$ is represented by a stochastic B-spline expansion with a multi-resolution structure. By selecting spline orders locally, the model adapts to target functions with heterogeneous smoothness across the domain. The number of basis functions at each order, $J_k$, is automatically adjusted through the Poisson prior, enabling data-driven model complexity. Because the model dimension changes in posterior simulation, we employ reversible jump MCMC (RJMCMC) of \citet{green1995reversible}. Implementation details are provided in \citet{park2023levy}, Section 4, Algorithm.

\begin{remark}
    The hierarchical specification in \eqref{eq:LABSprior} describes the general LABS prior. 
    To establish posterior concentration over Besov spaces, we later specialise this prior by imposing suitable, sample-size-dependent choices of the hyperparameters; see Section~\ref{subsec:posterior_contraction}.
\end{remark}

\section{Results}
\label{sec:results}

\subsection{Posterior Contraction}
This subsection briefly overviews posterior contraction of the Bayesian regression model, showing that the posterior distribution of the model concentrates around the true $\theta_0 = (f_0, \sigma_0)$. Let $\Theta$ denote the entire model space of the proposed model. By Bayes' rule, the posterior distribution is given by
\begin{equation*}
    \Pi(B \mid \mathbf{D}^{(n)}) :=
    \frac{\int_B p^{(n)}(\mathbf{D}^{(n)} \mid \theta)\, \Pi(d \theta)}{\int p^{(n)}(\mathbf{D}^{(n)} \mid \theta)\, \Pi(d \theta)},
\end{equation*}
where $B \subset \mathcal{M}$ is an arbitrary measurable set and $p^{(n)}$ is the density for the data $\bfD^{(n)}$.  
Our goal is to show that, for a suitable distance $d(\theta,\theta_0)$ between $\theta \in \Theta$ and the truth $\theta_0$, and for a contraction rate $\epsilon_n \to 0$,
\begin{equation*}
    \Pi\!\left( d(\theta, \theta_0) \le \epsilon_n \,\middle|\, \mathbf{D}^{(n)} \right)
    \rightarrow 1 \quad \text{in } \mathbb{P}_0^{(n)}\text{-probability.}
\end{equation*}
where $\bbP^{(n)}_0$ denotes the true distribution which generated $\bfD^{(n)}$ as in \eqref{eq:statModel}. We use as our distance the root average squared Hellinger distance $d_{n,H}$, defined on the product measure $P_\theta^{n} := \textstyle\bigotimes_{i=1}^n P_{\theta,i}$ associated with the observation-wise distributions $P_{\theta,i}$.  
Assume that each $P_{\theta,i}$ admits a density $p_{\theta, i}$.  
Given densities $p_{\theta_1,i}$ and $p_{\theta_2,i}$, the Hellinger distance is defined by $d_H(p,q) = \big( \textstyle\int (\sqrt{p} - \sqrt{q})^2 d\mu \big)^{1/2}$, and the root average squared Hellinger distance is
\begin{equation*}\label{eq:RootAvgSqrHelDist}
    d_{n,H}(\theta_1,\theta_2)
    := \sqrt{\frac{1}{n} \sum_{i=1}^n d_H^2(p_{\theta_1,i}, p_{\theta_2,i})}.
\end{equation*}

When the input domain is $[0,1]^d$ and the true function $f_0$ is $s$-regular (smooth), the minimax convergence rate under $L^2$ loss is known to be \citep{tsybakov2003introduction}
\begin{equation}    \label{eq:minimax_rate}
    \epsilon_n \asymp n^{-\frac{s}{2s + d}}.
\end{equation}
For $f_0$ belonging to a Besov class, \cite{donoho1998minimax} showed that the wavelet shrinkage estimator achieves the optimal convergence rate and adapts to the unknown smoothness of $f_0$.  
In the Bayesian framework, this corresponds to the posterior contraction rate.  
More recently, \cite{lee2022asymptotic} showed that Bayesian ReLU networks achieve a posterior contraction rate that is near-optimal relative to \eqref{eq:minimax_rate}, up to a logarithmic factor.

To establish posterior contraction, we employ a modified version of Theorem~4 in \citet{ghosal2007convergence}.
Let $\epsilon_n \to 0$ be a sequence satisfying $n\epsilon_n^2 \to \infty$, and let $\Theta_n \subset \Theta$ be an increasing sequence of subsets.  
If there exists $C > 2$ such that the following three conditions hold:
\begin{align}
    &\sup_{\epsilon>\epsilon_n}\log \mathcal{N}\!\left(\frac{\epsilon}{36},\, \Theta_n,\, d_{n,H}\right)
        \lesssim n\epsilon_n^2, \label{eq:EntropyCond}\\
    &\Pi_n(\Theta\setminus \Theta_n)
        = o\!\left(e^{-(C+2)n\epsilon_n^2}\right), \label{eq:TailBoundCond}\\
    &\Pi_n\!\big(B^\ast_{n,2}(\theta_0,\epsilon_n)\big)
        \ge e^{-C n\epsilon_n^2}
        \quad \text{for all large } n, \label{eq:KLSuppCond}
\end{align}
then for any sequence $M_n \to \infty$,
\begin{equation*}
    \Pi \!\left( \theta \in \Theta : d_{n,H}(\theta, \theta_0) > M_n \epsilon_n
    \,\middle|\, \mathbf{D}^{(n)} \right)
    \rightarrow 0, \quad n \to \infty,
\end{equation*}
in $\mathbb{P}_0^{(n)}$-probability.  
That is, the posterior distribution contracts toward $\theta_0 = (f_0, \sigma_0^2)$ at rate $\epsilon_n$.

Condition \eqref{eq:EntropyCond} is an entropy (log-covering number) bound ensuring that the model space is not excessively complex.  
Condition \eqref{eq:TailBoundCond} controls the prior mass outside the sieve $\Theta_n$, requiring it to decay exponentially.  
Condition \eqref{eq:KLSuppCond} guarantees that the prior assigns sufficient mass to Kullback–Leibler neighborhoods $B_{n,2}^\ast (\theta_0, \epsilon_n)$ of the true parameter which is defined as
\begin{equation*}
    B_{n,2}^\ast (\theta_0, \epsilon) := \left\{ \theta : \frac{1}{n} \sum_{i=1}^n K (p_{\theta_0, i}, \, p_{\theta, i} ) \leq \epsilon^2, \; \frac{1}{n} \sum_{i=1}^n V_{2,0} (p_{\theta_0, i}, \, p_{\theta, i} ) \leq \epsilon^2 \right\}.
\end{equation*}
The main task in the subsequent analysis is to identify a rate $\epsilon_n$ and a sieve $\Theta_n$ satisfying these three conditions simultaneously.

\subsection{Posterior Contraction of LABS in Besov Spaces}
\label{subsec:posterior_contraction}

We now establish posterior contraction of the LABS model \eqref{eq:LABSprior} in the Besov space $B^s_{p,q}([0,1])$. We begin by specifying the hyperparameters of the prior so that they depend on the sample size $n$. Specifically, let
\begin{enumerate}[label=(P\arabic*), ref=P\arabic*]
    \item\label{eq:BspDegreeCond} The set $S$ is a finite collection of spline degrees satisfying $k \ge 1$.
    \item\label{eq:InvGamCond} The hyperparameters of the inverse-gamma prior $\sigma^2 \sim \mathrm{Inv\text{-}Gam}(r/2, rR/2)$ satisfy $r > 0$ and $R > 0$.
    \item\label{eq:GamPriorCond} For the Poisson mean hyperprior $M_k \sim \mathrm{Gam}(a_k,b_k)$, we set
    \begin{align*}
        a_k &> 0, \\
        b_k &= b_n := e^{C_b(\log n)^2}, \quad \forall k\in S,
    \end{align*}
    for a constant $C_b>0$.
    \item\label{eq:NormPriorCond} The regression coefficients satisfy $\beta_{k,l}\sim N(0,\phi_k^2)$ with
    \begin{equation*}
        \phi_k=\phi_n:= e^{ C_\phi(\log n)^2 }, \quad \forall k\in S,
    \end{equation*}
    for a constant $C_\phi > 0.$
    \item\label{eq:UnifPriorCond} The knot prior is a truncated uniform distribution with minimum spacing constraint:
    \begin{equation*}
    \boldsymbol{\xi}_{k,l} \sim U(\mathcal{X}^{k+2}(\delta_n)), \quad
        \delta_n:= e^{-C_\delta(\log n)^2 },
    \end{equation*}
    where $\mathcal{X}^{k+2}(\delta_n)=\{\mathbf{z}\in\mathcal{X}^{k+2}: m(\mathbf{z})\ge \delta_n\}$ and $m(\mathbf{z})=\min_{1\le j\le k+1}|z_{(j+1)}-z_{(j)}|$ for $z_{(1)} < \cdots < z_{(k + 2)}$.
\end{enumerate}

The following theorem shows that, under previously mentioned setting, LABS achieves nearly optimal (up to logarithmic factors) adaptive posterior contraction rates over Besov classes with $s > \left( 1/p - 1/2 \right)_+$.
\begin{theorem} \label{eq:main_theorem}
Let the observations $\mathbf{D}^{(n)} = \{(x_i, y_i): i=1,\ldots,n\}$ follow model \eqref{eq:statModel} under Assumption \eqref{eq:RandDesigAssump}. Suppose the LABS model \eqref{eq:LABSprior} satisfy \eqref{eq:BspDegreeCond}–\eqref{eq:UnifPriorCond}.  
For any $f_0 \in B^s_{p,q}([0,1])$ satisfying
\begin{align} \label{eq:Besov_smoothness_cond}
   s>\left(\frac{1}{p}-\frac{1}{2}\right)_+,
\end{align}
if $S$ in \eqref{eq:BspDegreeCond} contains a degree $k^\star$ such that
\begin{equation}    \label{eq:Besov_model_degree_cond}
    s < \min{(k^\star, k^\star - 1 + \frac{1}{p})},
\end{equation}
then the posterior distribution concentrates around $\theta_0 = (f_0, \sigma_0)$ at rate
\begin{equation} \label{eq:rate}
        \epsilon_n = n^{-\frac{s}{2s+1}} (\log n)^2
\end{equation}
under the metric $d_{n,H}$, that is,
\begin{equation*}
    \Pi\!\left( (f, \sigma) : d_{n,H}(p_{f, \sigma}, \, p_{f_0, \sigma_0}) > M_n \epsilon_n
    \,\middle|\, \mathbf{D}^{(n)} \right) \to 0
\end{equation*}
in $\mathbb{P}_0^{(n)}$-probability as $n \to \infty$.
\end{theorem}

\begin{proof}
The proof is provided in the Appendix~\ref{sec:main_thm_proof}
\end{proof}

Under the same setup as in the theorem, adding a uniform boundedness assumption yields an analogous adaptive posterior contraction result under the $L^2(P_X)$ loss.

\begin{corollary}
Suppose that model \eqref{eq:statModel} satisfies Assumption
\eqref{eq:RandDesigAssump} together with
\eqref{eq:VarAssump} and \eqref{eq:UnifAssump}.
Assume the same prior specification and conditions as in
\textbf{Theorem~\ref{eq:main_theorem}},
except that the prior for
$\sigma^2$ in \eqref{eq:InvGamCond} is truncated to
$[\underline{\sigma}^2,\overline{\sigma}^2]$ for some constants
$0<\underline{\sigma}<\overline{\sigma}<\infty$.
Additionally, define
\begin{equation*}
    \operatorname{clip}_F (x) := \min \{F, \max \{ -F, x \} \}.
\end{equation*}
Then, for any sequence $\tilde M_n \to \infty$,
\begin{equation*}
    \Pi\!\left(
    (f,\sigma) :
    \|\operatorname{clip}_F \circ f-f_0\|_{L^2(P_X)} + |\sigma-\sigma_0|
    > \tilde{M}_n \epsilon_n
    \,\middle|\,
    \mathbf{D}^{(n)}
    \right)
    \to 0
\end{equation*}
in $\mathbb{P}_0^{(n)}$-probability as $n \to \infty$.
\end{corollary}

\begin{proof}
For any $f: \bbR\ \to \bbR$, define $f^F := \operatorname{clip}_F \circ f$.
Then $\| f^F \|_{L^\infty} \leq F$ and, by assumption, $\| f_0 \|_{L^\infty} \leq F$.
Let $C_F := F^2 + 8 \overline{\sigma}^2$.
By \textbf{Lemma B.1} of \citet{xie2020adaptive}, for $\sigma,\sigma_0 \in [\underline{\sigma},\overline{\sigma}]$,
\begin{equation*}
    \frac{1}{C_F}
    \left(
        \|f^F - f_0\|_{L^2(P_X)}^2 + 2|\sigma - \sigma_0|^2
    \right)
    \le
    d_{n,H}^2(p_{f^F,\sigma}, p_{f_0,\sigma_0})
    \le
    d_{n,H}^2(p_{f,\sigma}, p_{f_0,\sigma_0})
    .
\end{equation*}
Therefore,
\begin{equation*}
    \|f^F - f_0\|_{L^2(P_X)} + |\sigma-\sigma_0|
    \leq
    \sqrt{\frac{3C_F}{2}}\,
    d_{n,H}(p_{f,\sigma},p_{f_0,\sigma_0}),
\end{equation*}
where we used $a + b \leq \sqrt{\frac{3}{2}} \cdot \sqrt{a^2 + 2 b^2}$ for $a , b > 0$.
Now let $\tilde{M}_n \to \infty$ be arbitrary and define $M_n := \sqrt{\frac{2}{3C_F}}\, \tilde{M}_n$.
Then
\begin{align*}
    &
    \Pi\!\left(
    (f,\sigma) :
    \| \operatorname{clip}_F \circ f-f_0\|_{L^2(P_X)} + |\sigma-\sigma_0|
    > \tilde{M}_n \epsilon_n
    \,\middle|\,
    \mathbf{D}^{(n)}
    \right) \\
    &\le
    \Pi\!\left(
    (f,\sigma) :
    d_{n,H}(p_{f,\sigma}, p_{f_0,\sigma_0})
    > M_n \epsilon_n
    \,\middle|\,
    \mathbf{D}^{(n)}
    \right).
\end{align*}
The right-hand side converges to zero in $\mathbb{P}_0^{(n)}$-probability by \textbf{Theorem~\ref{eq:main_theorem}}.
\end{proof}

\section{Simulation Study}
\label{sec:simulation}

In this section, we present simulation results that support the posterior contraction property of LABS established in Section~\ref{subsec:posterior_contraction}. The simulations are based on the test functions introduced in \citet{donoho1994ideal}, namely the Blocks function $f_1$, the Bumps function $f_2$, the Heavisine function $f_3$, and the Doppler function $f_4$.
These functions exhibit patterns such as local discontinuities or abrupt changes in the derivative at specific locations, and are therefore suitable examples for empirically validating the theoretical results of this work.

\begin{enumerate}
    \item Blocks
        \begin{align*}
            f_1(x) &= \sum_j h_j \, K (x - t_j), \quad K(t) = \frac{1 + \mathrm{sgn}(t)}{2}, \\
            (t_j) &= (0.1, 0.13, 0.15, 0.23, 0.25, 0.40, 0.44, 0.65, 0.76, 0.78, 0.81), \\
            (h_j) &= (4, -5, 3, -4, 5, -4.2, 2.1, 4.3, -3.1, 2.1, -4.2).
        \end{align*}
    \item Bumps
        \begin{align*}
            f_2(x) &= \sum_j h_j \, K \left( \frac{x - t_j}{w_j} \right), \quad K(t) = (1 + |t|)^{-4}, \\
            (t_j) &= (0.1, 0.13, 0.15, 0.23, 0.25, 0.40, 0.44, 0.65, 0.76, 0.78, 0.81), \\
            (h_j) &= (4, 5, 3, 4, 5, 4.2, 2.1, 4.3, 3.1, 2.1, 4.2), \\
            (w_j) &= (0.005, 0.005, 0.006, 0.01, 0.01, 0.03, 0.01, 0.01, 0.005, 0.008, 0.005).
        \end{align*}
    \item HeaviSine
        \begin{equation*}
            f_3 (x) = 4 \sin(4 \pi x) - \mathrm{sgn}(x - 0.3) - \mathrm{sgn}(0.72 - x).
        \end{equation*}
    \item Doppler
        \begin{equation*}
            f_4(x) = \sqrt{x (1-x)} \, \sin\left\{ \frac{2.1 \pi }{x + 0.05} \right\}.
        \end{equation*}
\end{enumerate}
Here, Blocks, Bumps and Heavisine are bounded variation functions, and $f_1, f_2, f_3 \in B^1_{1, \infty}([0, 1])$. In contrast, the Doppler function $f_4$ exhibits spatial inhomogeneity with rapid oscillations near the origin. Consequently, $f_4$ is typically considered to lie in $B^s_{p, \infty}([0, 1])$ with $s < 1$ for $p \in [1, \infty]$. Note that all four functions fall within our theoretical guarantees.
\begin{figure}[tbp] 
     \centering
     \begin{subfigure}[b]{0.48\textwidth}
         \centering
         \includegraphics[width=0.8\textwidth]{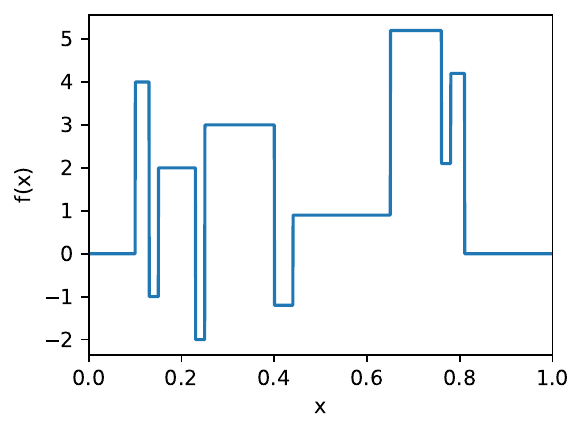}
         \caption{$f_1$: Blocks}
         \label{fig:blocks}
     \end{subfigure}
     \hfill
     \begin{subfigure}[b]{0.48\textwidth}
         \centering
         \includegraphics[width=0.8\textwidth]{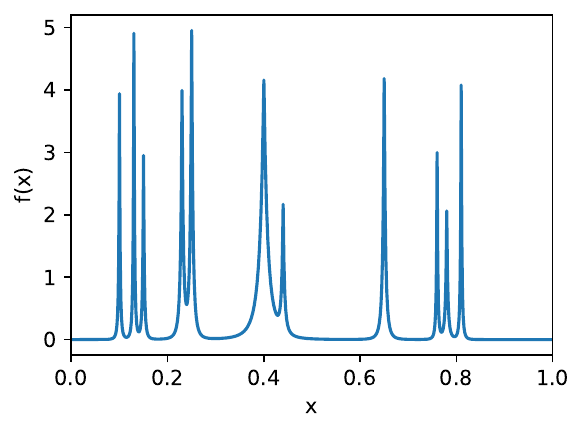}
         \caption{$f_2$: Bumps}
         \label{fig:bumps}
     \end{subfigure}

     \vspace{10pt}

     \begin{subfigure}[b]{0.48\textwidth}
         \centering
         \includegraphics[width=0.8\textwidth]{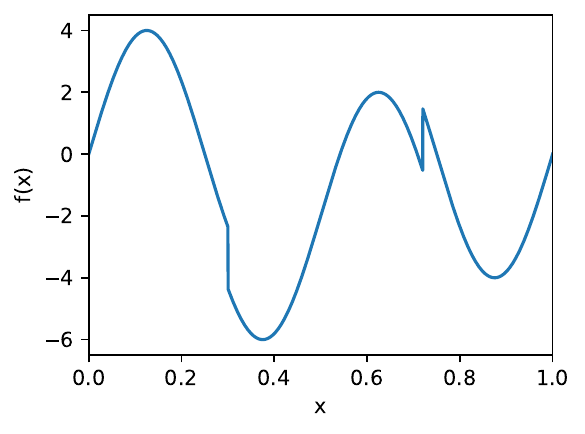}
         \caption{$f_3$: HeaviSine}
         \label{fig:heavisine}
     \end{subfigure}
     \hfill
     \begin{subfigure}[b]{0.48\textwidth}
         \centering
         \includegraphics[width=0.8\textwidth]{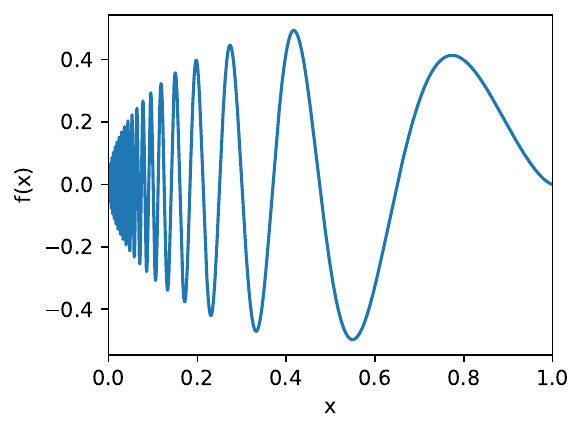}
         \caption{$f_4$: Doppler}
         \label{fig:doppler}
     \end{subfigure}
     
     \caption{Four test functions from \citet{donoho1994ideal} used in the simulation.}
     \label{fig:test_functions}
\end{figure}
For each function, we generated 100 simulated datasets $(x_i, y_i)_{i=1}^n$ with $n = 128, 1024, \allowbreak 8192$. The data were generated under a random design:
\begin{equation*}
        x_i \stackrel{iid}{\sim} U([0,1]), \quad    
        y_i \mid x_i \stackrel{ind}{\sim} N(f_j(x_i), \sigma_0^2), \quad j = 1, \ldots, 4.
\end{equation*}
The noise level $\sigma_0$ was chosen according to three root signal-to-noise ratio (RSNR) settings, RSNR $= 3, 5, 10$:
\begin{equation*}
        \mathrm{RSNR} := \sqrt{\frac{\int_{\mathcal{X}} (f(x) - \bar{f})^2}{\sigma^2}},
\end{equation*}
where $\bar{f} := \int_0^1 f(x) \, dx$. We standardised the observations $y_i$ prior to analysis. Thus, the corresponding values of $\sigma_0$ for each RSNR are given by the reciprocal of RSNR, that is, $\sigma_0 = \frac{1}{3}, \frac{1}{5}, \frac{1}{10}$.

We then compared LABS with the following competing methods on the simulated datasets:
\begin{enumerate}
    \item (EBW) Empirical Bayes selection of wavelet thresholds with a Laplace prior and Daubechies ``la8'' wavelets \citep{johnstone2005empirical}.
    \item (GP) Gaussian process regression with an exponential kernel \citep{williams2006gaussian}.
    \item (LOESS) Local polynomial regression of degree 1 with automatic smoothing parameter selection \citep{cleveland1979robust}.
\end{enumerate}

All methods were implemented in the \texttt{R} language \citep{rcoreteam2025r}. We used the \texttt{Ebayesthresh} package \citep{silverman2005ebayesthresh} for EBW, the \texttt{kernlab} package \citep{karatzoglou2024kenlab} for GP, and the \texttt{fANCOVA} package \citep{wang202fancova} for LOESS.

The hyperparameters of LABS were chosen to satisfy the conditions guaranteeing posterior contraction in Section~\ref{subsec:posterior_contraction}. The multiplicative constants for these hyperparameters were tuned appropriately and are summarised in \hyperref[tab:hparamsTuning]{Table~\ref{tab:hparamsTuning}}. Note that the multiplicative constants were allowed to vary depending on the simulation settings. For the gamma prior in \eqref{eq:GamPriorCond}, we set $b_n = e^{C_b (\log n)^2 }$, while for the normal prior in \eqref{eq:NormPriorCond}, we used $\phi_n = e^{C_\phi ( \log n )^2}$. The knot spacing in \eqref{eq:UnifPriorCond} was set to $\delta_n = e^{- (\log n)^2}$.

\begin{table}[tbp]
    \centering
        \begin{tabular}{ c  c  c  c  c  c }
        \toprule
             $n$ & $S$ & $a$ & $b_n$ & $\phi_n$ & $\delta_n$    \\ 
        \midrule
             $128$ & $\{ 1, 2 \}$ & $1$ & $1.002357$ & $7.278045$ & $5.966965 \times 10^{-11}$  \\
             $1024$ & $\{ 1, 2 \}$ & $1$ & $1.004816$ & $10.397208$ & $1.362043 \times 10^{-21}$    \\
             $8192$ & $\{ 1, 2 \}$ & $1$ & $1.008153$ & $13.516370$ & $5.454844 \times 10^{-36}$  \\
        \bottomrule
        \end{tabular}
        \caption{\label{tab:hparamsTuning} Hyperparameter settings of the LABS model used in the simulation study. The sample sizes are $n = 128, 1024, 8192$
        The B-spline degree set is $S = \{1, 2 \}$, corresponding to linear (degree $1$) and quadratic (degree $2$) spline kernels.}
\end{table}
To compare the four methods, we used the mean squared error (MSE) between the true function values $f(x_i)$ and the fitted values $\hat{f}(x_i)$ as the performance metric. For ease of visual comparison, we considered the logarithm of MSE.

\begin{equation*}
        \mathrm{MSE} = n^{-1} \sum_{i=1}^n (f_0(x_i) - \hat{f}(x_i))^2.
\end{equation*}

\hyperref[fig:boxplot_n128]{Figure~2}, \hyperref[fig:boxplot_n1024]{Figure~3}, and \hyperref[fig:boxplot_n8192]{Figure~4} display boxplots of the MSE values for $n = 128, 1024, 8192$, respectively, based on 100 simulation replicates. The LABS hyperparameters were set according to \hyperref[tab:hparamsTuning]{Table~\ref{tab:hparamsTuning}}. Under the specification $\bxi_{k, l} \sim U(\mathcal{X}^{k +2}(\delta_n))$ in \eqref{eq:UnifPriorCond}, the knot domain was taken to be $\mathcal{X} = [0, 1]$.

Overall, LABS outperforms the competing methods in most scenarios. In particular, in the low-sample-size setting, LABS yields the smallest prediction error among all methods. \hyperref[fig:boxplot_n128]{Figure~2} shows that for $n=128$, LABS achieves the lowest log MSE across all test functions and RSNR levels.

\begin{figure}[tbp]
        \centering
        \includegraphics[width=\linewidth]{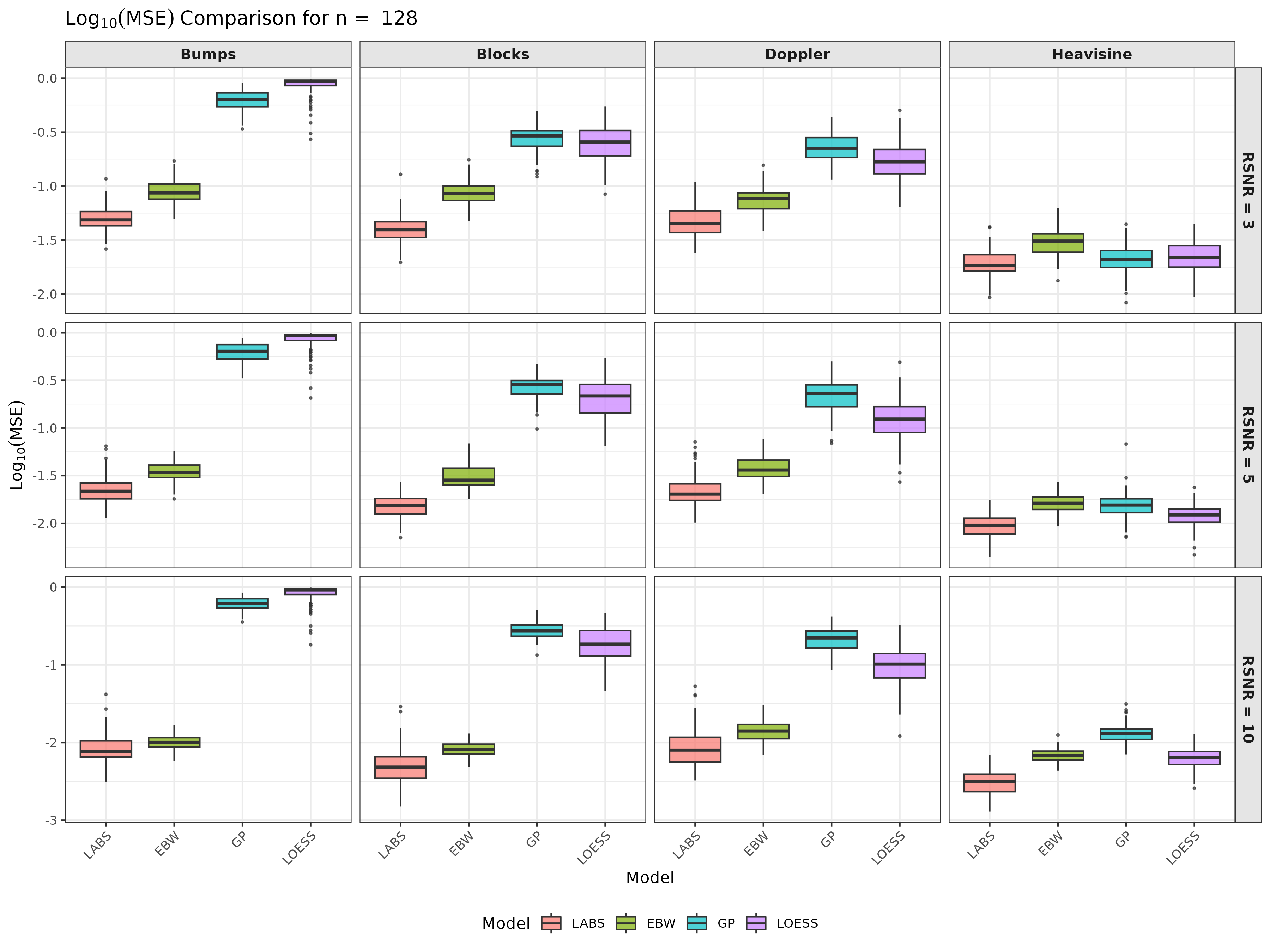}
        \caption{MSE boxplots for $n = 128$ based on 100 simulation replicates.}
        \label{fig:boxplot_n128}
\end{figure}

For $n=1024$ and $n=8192$, as shown in \hyperref[fig:boxplot_n1024]{Figure~3} and \hyperref[fig:boxplot_n8192]{Figure~4}, LABS also performs well overall, although there exist specific scenarios in which EBW slightly outperforms LABS.

In high-noise scenarios, LABS tends to perform better than EBW. For instance, in the first row of \hyperref[fig:boxplot_n1024]{Figure~3}, corresponding to RSNR $= 3$ and $\sigma \approx 0.333$, LABS shows superior performance across all test functions. However, in the second row with RSNR $= 5$ and $\sigma = 0.2$, LABS is slightly inferior to EBW for the Bumps and Doppler test functions.

Furthermore, the results suggest that LABS provides accurate estimates for functions with diverse discontinuity patterns. These include functions with abrupt jumps at specific points: $(t_j)$ for Blocks, and $x = 0.3, 0.72$ for Heavisine. This behaviour is clearly observed in the Blocks and Heavisine tests in \hyperref[fig:boxplot_n1024]{Figure~3} and \hyperref[fig:boxplot_n8192]{Figure~4}, where LABS generally attains lower MSE than the competing methods across all RSNR settings.
    
\begin{figure}[tbp]
        \centering
        \includegraphics[width=\linewidth]{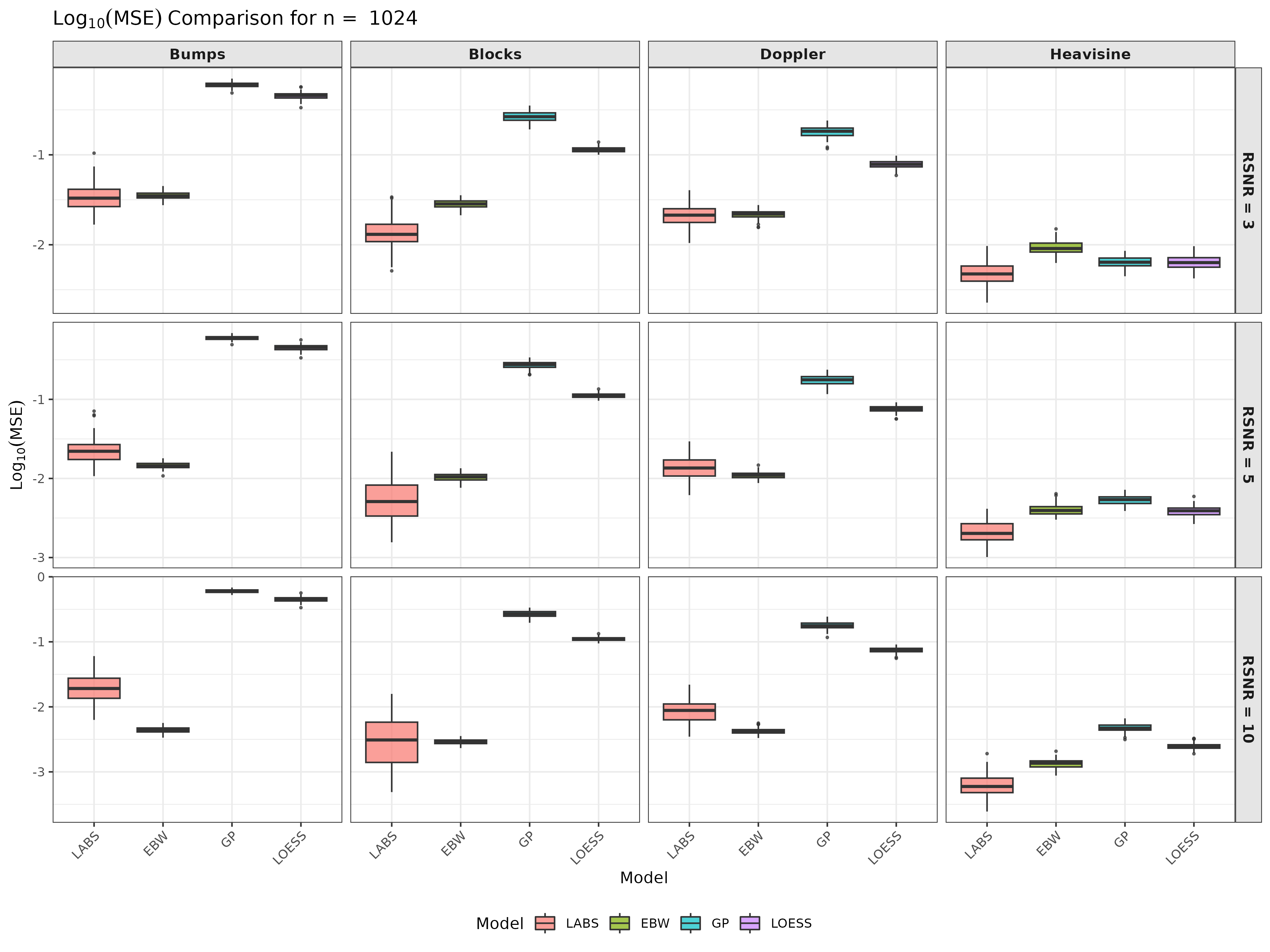}
        \caption{MSE boxplots for $n = 1024$ based on 100 simulation replicates.}
        \label{fig:boxplot_n1024}
\end{figure}
    
\begin{figure}[tbp]
        \centering
        \includegraphics[width=\linewidth]{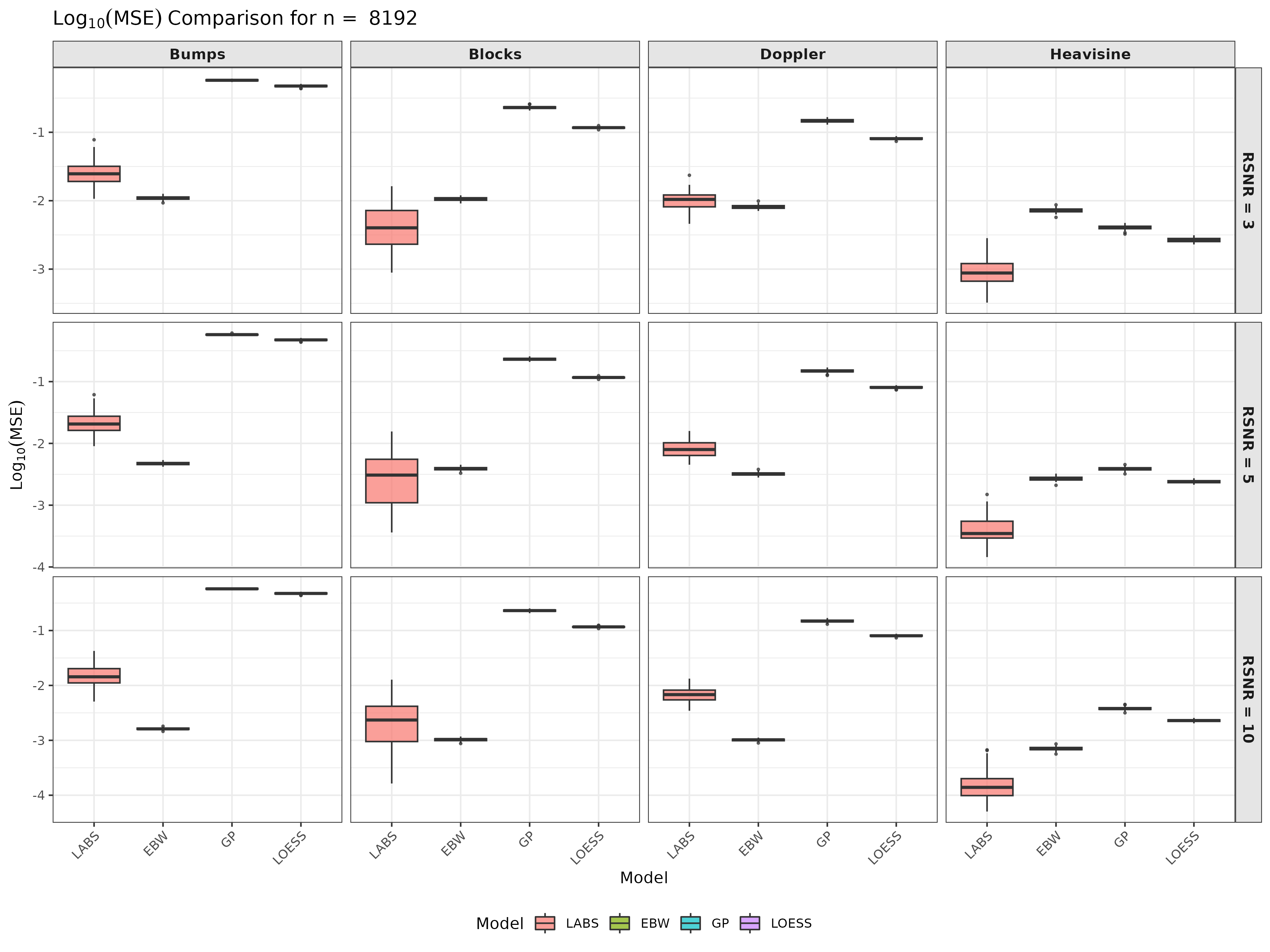}
        \caption{MSE boxplots for $n = 8192$ based on 100 simulation replicates.}
        \label{fig:boxplot_n8192}
\end{figure}

\section{Concluding Remarks}
\label{sec:concluding}

In this study, we investigated posterior contraction for a fully Bayesian regression model based on random B-splines. The true regression function was assumed to belong to a Besov space, which includes functions that may exhibit irregular or nonsmooth behaviour. The proposed model employs B-spline bases of multiple orders, allowing local adaptation of smoothness. This property enables the model to effectively estimate functions in Besov spaces that may contain discontinuities or spatially heterogeneous features. 

Previous work on Bayesian regression models using random B-splines for estimating functions in Besov spaces has been limited, and this study aimed to fill that gap. We established that the proposed LABS model achieves nearly optimal adaptive posterior contraction rates—up to logarithmic factors—over Besov spaces with smoothness $s > \left( 1/p - 1/2 \right)_+$.
Furthermore, we proved that the adaptive contraction holds under the root average squared Hellinger distance, and that the same rate applies under the $L^2$ loss when a uniform boundedness assumption is imposed. Additionally, empirical experiments confirmed that the theoretical convergence behaviour is consistently observed in practice. A natural direction for future work is to extend the present theory to multivariate regression settings and analyze the corresponding posterior contraction rates.

\appendix
\addcontentsline{toc}{section}{Appendix}

\section*{Appendix}

\section{Preliminaries}

This appendix provides the proof of the main theorem stated in the main paper.
To avoid repetition, we introduce only the definitions and notation used exclusively in this appendix. All other definitions and notation are as given in Section~\ref{sec:preliminaries} of the main paper.

\subsection{Definitions and Notation}

Let $\calX := [-A, 1 + A],$ for some $A > 0$. For a real-valued function $f$ defined on $\calX$, let $[x_1, \ldots, x_{r + 1}]f$ denote the $r$-th order divided difference. That is,
$[x_1, \ldots, x_{r + 1}]f = ([x_2, \ldots, x_{r + 1}]f - [x_1,\ldots, x_{r}]f) / (x_{r + 1} - x_1)$.
We also define $[x_i]f = f(x_i)$. If $x_1 = \cdots = x_{r + 1}$, then $[x_1, \ldots, x_{r + 1}]f = f^{(r)}(x_1) / r!$, which corresponds the $r$th-derivative of $f$ at $x_1$.

\subsection{Sieve Construction}
Here, we present the construction of a sieve which will be used in the proof of the main theorem. The sieve $\Theta_n$ is constructed as follows.  
For $\boldsymbol{J} = (J_k)_{k \in S} \in \mathbb{N}_0^{|S|}$, $B > 0$, and $\delta > 0$, define
\begin{align}   \label{eq:Sieve}
\begin{split}
    \calC (\bJ, B) &:= \{ \bbeta = (\beta_{k,l}) : \| \bbeta \|_\infty \leq B \},    \\
    \calX_{\xi} (\bJ, \delta) &:= \{ \bxi = (\bxi_{k,l}) : m (\bxi_{k,l}) \geq \delta, \; \bxi_{k,l} \in \calX^{k + 2} \; \text{for all} \; k,l \},   \\
    \mathcal{F}(\bJ,B,\delta)
    &:= \Big\{ f= \sum_{k \in S} \sum_{l = 1}^{J_k} \beta_{k,l} B_k ( \, \cdot \, ; \bxi_{k,l}) : \bbeta \in \calC(\bJ, B), \; \bxi \in \calX_\xi (\bJ, \delta) \Big\}.
\end{split}
\end{align}
Here, $m (\bfz) = \min_{1 \leq j \leq k + 1} |z_{(j + 1)} - z_{(j)}|$ for $\bfz \in \calX^{k + 2}$ and $z_{(j)}$ is the $j$-th smallest component of $\bfz$. For $J \in \mathbb{N}_0$, define
\begin{align}   \label{eq:Sieve_param}
\begin{split}
    \mathcal{F}(J,B,\delta) &:= \bigcup_{\boldsymbol{J}:\ \sum_{k\in S}J_k \le J} \mathcal{F}(\boldsymbol{J},B,\delta), \\
    \mathcal{F}_n &:= \mathcal{F}(J_n, B_n, \delta_n), \\
    \Theta_n &:= \mathcal{F}_n \times [\underline{\sigma}_n, \overline{\sigma}_n],
\end{split}
\end{align}
where the sequences are chosen as
\begin{equation}\label{eq:SieveCond}
\begin{matrix}
    J_n := \lceil C_J \, n^{\frac{1}{2s+1}} (\log n)^2 \rceil,
    & B_n := C_B \, \phi_n \, n^{\frac{1}{2(2s + 1)}} (\log n)^2,   \\
    \delta_n := \exp\{-C_\delta(\log n)^2\},
    &\underline\sigma_n := C_{\underline\sigma} \, n^{- \frac{1}{2(2s + 1)}} (\log n )^{-2}, \\
    \overline\sigma_n := \exp{C_{\overline\sigma} \, n^{\frac{1}{2s + 1}} (\log n )^{4}}
\end{matrix}
\end{equation}
for constants $C_J,C_B,C_\delta, C_{\underline\sigma}, C_{\overline\sigma}>0$,
which will be specified later in the proof of Proposition~\ref{eq:tail}.
The choices in \eqref{eq:SieveCond} are designed to ensure that the three sufficient
conditions in \cite{ghosal2007convergence} are satisfied simultaneously.
The model space is defined as
\begin{align*}
    \mathcal{F} &= \bigcup_{J,B,\delta} \mathcal{F}(J,B,\delta), \\
    \Theta &= \mathcal{F} \times (0,\infty).
\end{align*}

\section{Auxiliary Lemmas}

The following lemma is adapted from \textbf{Lemma 1} of \citet{belitser2014adaptive}, where the proof is given for functions $f:[0,1] \rightarrow \bbR$. Here we modify it to the LABS model setting of the present paper, namely $f:\calX \rightarrow \bbR$. Although the proof is essentially identical, we include it for the sake of completeness.

\begin{lemma} \label{lem:divdiff}
	Let $(\xi_1, \ldots, \xi_{r + 1}) \in \calX^{r + 1}$ with $r \geq 2$, and let $i \in \{1, \ldots, r \}$. Suppose that $\xi_{i + 1} - \xi_{i} = 0$ and that $\xi_{\nu + 1} - \xi_{\nu } \geq \delta > 0$ for all $\nu \in \{1, \ldots, i - 1, i + 1, \ldots, r \}$. For a fixed $x \in \calX$, define $f(y) = (x - y)_+^{q - 1}$ for $y \in \calX$, with $q \geq 2$. Then, for all $x \neq \xi_i$ and $\delta \leq \frac{2(1 + 2A)}{q - 1}$, the following holds:
	\begin{equation*}
		| [\xi_1, \ldots, \xi_{r + 1} ] f | \leq \frac{4 (1 + 2A)^{q-1}}{\delta^r}.
	\end{equation*}
\end{lemma}

\begin{proof}
	We have $| f^\prime(y)| = (q - 1)(x - y)^{q-2}_+ \leq (q - 1)(1 + 2A)^{q-2} \leq 2(1 + 2A)^{q - 1} / \delta$ for $x \neq y$, since $x, y \in \calX = [-A, 1 + A]$, $q \geq 2$, and $\delta \leq \frac{2(1 + 2A)}{q - 1}$. Now, when $\nu = i - 1$, we have
	$|[\xi_{\nu + 1}, \xi_{\nu + 2}] f | = |f^\prime(\xi_{\nu + 1}) | \leq 2 (1 + 2A)^{q - 1} /\delta$.
	On the other hand, when $\nu \neq i - 1$, we obtain
	$| [\xi_{\nu + 1}, \xi_{\nu + 2}]f | = |f(\xi_{\nu + 2}) - f(\xi_{\nu + 1}) | / |\xi_{\nu + 2} - \xi_{\nu + 1} | \leq 2(1 + 2A)^{q-1} / \delta$.
	Therefore, $|[\xi_{\nu + 1}, \xi_{\nu + 2} ]f | \leq 2(1 + 2A)^{q - 1} / \delta$ for all $x \neq \xi_{i}$. \\

	Next, we apply \textbf{Theorem 2.56} of \citet{Schumaker2007spline}. For $j = 2, \ldots, r$, define $\gamma_j := \min_{\nu = 1, \ldots, k+ 2 - j} |\xi_{\nu + j} - \xi_{\nu}| \geq (j - 1)\delta$. Then, by the theorem,
	\begin{equation*}
		|[\xi_1, \ldots, \xi_{r + 1}]f| \leq \sum_{\nu=0}^{r - 1} {\binom{r - 1}{\nu}} \frac{[\xi_{\nu + 1}, \xi_{\nu + 2}]f}{\gamma_2 \cdots \gamma_{k + 1}} \leq \frac{2^r (1 + 2A)^{q-1}}{(r - 1)! \delta^{r}} \leq \frac{4 (1 + 2A)^{q-1}}{\delta^r}.
	\end{equation*}
	This completes the proof.
\end{proof}

\begin{lemma}   \label{eq:knot_complexity}
	Consider two functions
    \begin{equation*}
    f_{\bbeta, \bxi} = \sum_{k \in S} \sum_{l = 1}^{J_k} \beta_{k,l} B_k (\cdot \, ; \, \bxi_{k,l}), \quad f_{\bbeta^\prime, \bxi^\prime} = \sum_{k \in S} \sum_{l = 1}^{J_k} \beta_{k,l}^\prime B_k(\cdot \, ; \, \bxi_{k,l}^\prime ).
    \end{equation*}
    For $B > 0$ and $\delta > 0$, suppose that $\bbeta \in \bbR^{\sum_{k \in S} J_k}$ satisfies $\|\bbeta\|_{\infty} \leq B$, and that $\bxi, \bxi^\prime \in \calX^{\sum_{k \in S} J_k(k + 2)}$ satisfy $m(\bxi_{k,l}) \geq \delta$ and $m(\bxi^\prime_{k,l}) \geq \delta$ for all $k,l$. Then,
	\begin{equation*}
		\| f_{\bbeta, \bxi} - f_{\bbeta, \bxi^\prime} \|_{L^{\infty}(\calX)} \leq L (B, \bJ, \delta; \overline{k} ) \| \bxi - \bxi^\prime \|_\infty
	\end{equation*}
	holds for all $\delta \leq 2(1 + 2A) / \overline{k}$, where
	$L (B, \bJ, \delta; \overline{k}) := 4 |S| (1 + 2A)^{\overline{k}} (\overline{k} + 2) B \big( \textstyle\sum_{k \in S} J_k \big) (\delta \wedge 1)^{-(\overline{k} + 1) }$
	and $\overline{k} = \textstyle\max_{k \in S} k.$
\end{lemma}

\begin{proof}
	The following inequality holds:
	\begin{align*}
		\|f_{\bbeta, \bxi} - f_{\bbeta, \bxi^\prime} \|_{L^\infty(\calX)}
		&\leq \sum_{k \in S} \sum_{0 < l \leq J_k} |\beta_{k,l}| \times \| B_k(\cdot; \bxi_{k,l}) - B_k(\cdot; \bxi^\prime_{k,l}) \|_{L^\infty(\calX)}	\\
		&\leq B \sum_{k \in S} \sum_{0 < l \leq J_k} \| B_k(\cdot; \bxi_{k,l}) - B_k(\cdot; \bxi^\prime_{k,l}) \|_{L^\infty(\calX)}.
	\end{align*}
	Define new knots by
	\begin{equation*}
		\bzeta^i_{k,l} := (\xi^\prime_{k,l,1}, \ldots, \xi^\prime_{k,l,i}, \xi_{k,l,i+1}, \ldots, \xi_{k,l,k+2}), \; \forall i \in \{0, 1, \ldots, k+2\}.
	\end{equation*}
	That is, $\bzeta^0_{k,l} = \bxi_{k,l}$ and $\bzeta^{k + 2}_{k,l} = \bxi^\prime_{k,l}$. Expanding the triangle inequality in a telescoping manner yields
	\begin{align*}
		\| B_k(\cdot; \bxi_{k,l}) - B_k(\cdot; \bxi^\prime_{k,l}) \|_{L^\infty(\calX)} &\leq \sum_{i = 0}^{k + 1} \| B_k(\cdot; \bzeta^{i}_{k,l}) - B_k(\cdot; \bzeta^{i+1}_{k,l}) \|_{L^\infty(\calX)}	\\
		&\leq (k + 2) \max_{0 \leq i \leq k + 1} \| B(\cdot; \bzeta^{i}_{k,l}) - B(\cdot; \bzeta^{i+1}_{k,l}) \|_{L^\infty(\calX)}.
	\end{align*}
	We note that
	$\|\bzeta^i_{k,l} - \bzeta^{i+1}_{k,l} \|_\infty = |\zeta^{i+1}_{k,l, i+1} - \zeta^{i}_{k,l,i+1} | = | \xi^\prime_{k,l,i+1} - \xi_{k,l,i+1} | \leq \|\bxi - \bxi^\prime\|_\infty$
	for all $i \in \{0, \ldots, k + 1 \}$. According to \textbf{Theorem 4.27} of \cite{Schumaker2007spline}, the derivative of $B(x ; \bzeta^i_{k,l})$ with respect to a knot is of the form of a divided difference satisfying the conditions of \textbf{Lemma~\ref{lem:divdiff}}. Hence, applying the lemma with $r, q = k + 1$, for fixed $x \in \calX$, $x \neq \xi_{k,l,i+1}$, and $\delta \leq 2(1 + 2A) / \overline{k}$, we obtain
	\begin{align*}
		|B_k(x ; \bzeta^i_{k,l}) - B_k(x ; \bzeta^{i+1}_{k,l}) | &\leq | \xi^\prime_{k,l,i+1} - \xi_{k,l, i+1} | \times \sup_{\xi_{k,l,i+1} \in (-A, 1 + A)} \left| \frac{\partial B(x ; \bzeta^{i}_{k,l})}{\partial \xi_{k,l,i+1}} \right|   \\
        &\leq \frac{4 (1 + 2A)^{k} \| \bxi - \bxi^\prime \|_\infty}{\delta^{k+1}}.
	\end{align*}
	Since B-splines of degree $k \geq 1$ are continuous, the same bound holds for $x = \xi_{k, l, i + 1}$. Substituting this back into the original inequality yields
	\begin{align*}
		\|f_{\bbeta, \bxi} - f_{\bbeta, \bxi^\prime} \|_{L_\infty(\calX)} &\leq B \sum_{k \in S} \sum_{0 < l \leq J_k} (k + 2) \frac{4 (1 + 2A)^{k} \|\bxi - \bxi^\prime \|_\infty}{\delta^{k + 1}}	\\
		&\leq 4 (1 + 2A)^{\overline{k}} (\overline{k} + 2) B \left( \sum_{k \in S} J_k \right) (\delta \wedge 1)^{- (\overline{k} + 1)} \| \bxi - \bxi^\prime \|_\infty.
	\end{align*}
	This completes the proof.
\end{proof}

\begin{lemma}   \label{eq:complexity}
	Consider two functions
    \begin{equation*}
    f_{\bbeta, \bxi} = \sum_{k \in S} \sum_{l = 1}^{J_k} \beta_{k,l} B_k (\cdot \, ; \, \bxi_{k,l}), \quad f_{\bbeta^\prime, \bxi^\prime} = \sum_{k \in S} \sum_{l = 1}^{J_k} \beta_{k,l}^\prime B_k(\cdot \, ; \, \bxi_{k,l}^\prime ),
    \end{equation*}
    suppose that $\bbeta, \bbeta^\prime \in \bbR^{\sum_{k \in S} J_k}$ satisfy $\| \bbeta \|_\infty \leq B$ and $\| \bbeta^\prime \|_\infty \leq B$, and that $\bxi = (\bxi_{k,l})$ and $\bxi^\prime = (\bxi^\prime_{k,l}) \in \calX^{\sum_{k \in S}J_k (k + 2)}$ satisfy $m(\bxi_{k,l}) \geq \delta$ and $m(\bxi^\prime_{k,l}) \geq \delta$ for all $k, l$. If $\delta \leq 2( 1 + 2A) / \overline{k}$, then
	\begin{equation*}
		\| f_{\bbeta, \bxi} - f_{\bbeta, \bxi^\prime} \|_{L^{\infty}(\calX)} \leq \left(\sum_{k \in S} J_k \right) \| \bbeta - \bbeta^\prime \|_\infty + L(B, \bJ, \delta; \overline{k}) \|\bxi - \bxi^\prime \|_\infty
	\end{equation*}
	where
	$L (B, \bJ, \delta ; \overline{k} ):= 4 (1 + 2A)^{\overline{k}} (\overline{k} + 2) B \big( \textstyle\sum_{k \in S} J_k \big) (\delta \wedge 1)^{- (\overline{k} + 1) }$
	and $\overline{k} = \textstyle\max_{k \in S} k.$
\end{lemma}

\begin{proof}
	By the triangle inequality and \textbf{Lemma~\ref{eq:complexity}},
	\begin{align*}
		\| f_{\bbeta, \bxi} - f_{\bbeta^\prime, \bxi^\prime} \|_{L^{\infty}(\calX)} &\leq \| f_{\bbeta, \bxi} - f_{\bbeta^\prime, \bxi} \|_{L^{\infty}(\calX)} + \| f_{\bbeta^\prime, \bxi} - f_{\bbeta^\prime, \bxi^\prime} \|_{L^{\infty}(\calX)}  \\
        &\leq \sum_{k \in S} \sum_{l = 1}^{J_k} |\beta_{k,l} - \beta^\prime_{k,l}| + L(B, \bJ, \delta; \overline{k}) \cdot \| \bxi - \bxi^\prime \|_\infty  \\
        &\leq \left( \sum_{k \in S} J_k \right) \| \bbeta - \bbeta^\prime \|_\infty + L(B, \bJ, \delta ; \overline{k}) \cdot \| \bxi - \bxi^\prime \|_\infty.
	\end{align*}
	This completes the proof.
\end{proof}

\begin{lemma} \label{eq:apprxlem}
	Let the true function satisfy $f_0 \in B^s_{p,q}([0,1])$. If $s > 0$ satisfies conditions $s > ( \frac{1}{p} - \frac{1}{2} )_+$ and $s < \min \{k, k - 1 + \frac{1}{p} \}$, then the following holds. There exists a cardinal B-spline approximation $\hat{f}_n$ of degree $k$ satisfying
	\begin{equation*}
		\|\hat{f}_{n} - f_0\|_{L^2} \lesssim n^{- \frac{s}{2s + 1}}
	\end{equation*}
	for all $n$, such that $\hat{f}_n$ belongs to the set $\hat\calF_n$ defined by
	\begin{equation*}
		\hat\calF_n := \left\{ f = \sum_{l=1}^{\hat{J}_n} \beta_l B_k(\cdot ; \bxi_{k,l}) : \| \bbeta \|_\infty \leq \hat{B}_n, \; m(\bxi_{k,l}) \geq \hat{\delta}_n, \; \forall l \right\}.
	\end{equation*}
	The quantities $\hat{J}_n$, $\hat{B}_n$, and $\hat{\delta}_n$ are defined as
	\begin{align}  \label{eq:approximation_space}
    \begin{split}
		\hat{J}_n &\asymp n^{\frac{1}{2s + 1}}	\\
		\hat{B}_n &\asymp n^{\frac{1}{2s + 1} (\nu^{-1} + 1) \left( \frac{1}{p} - s \right)_+}	\\
		\hat{\delta}_n &\asymp 2^{- \left( \frac{C_1 + \nu^{-1}}{2s + 1} \log{n} \right)}.
    \end{split}
	\end{align}
	Here, the constant $C_1 > 0$ appears in \textbf{Lemma 2} of \citet{suzuki2018adaptivity} and does not depend on $n$, and $\nu = (s - \omega) / (2 \omega)$, where $\omega = (1 / p - 1 / 2)_+$.
\end{lemma}

\begin{proof}
	In \textbf{Lemma 2} of \citet{suzuki2018adaptivity}, set $r = 2$ and $N = N_n := C_N \, n^{\frac{1}{2s + 1}}$. Then, there exists an approximation $\hat{f}_n$ based on cardinal B-spline bases of degree $k$ satisfying $s < \min\{k, k - 1 + 1/p \}$ such that
	\begin{equation*}
		\|f_0 - \hat{f}_{n} \|_{L^2} \lesssim n^{- \frac{s}{2s + 1}} \|f_0\|_{B^s_{p,q}([0,1])}
	\end{equation*}
	for all sufficiently large $n$. Choose $C_N > 0$ so that the above inequality holds for all $n$. Let $\hat{J}_n$ denote the number of basis functions in $\hat{f}_n$. By \textbf{Lemma 2}, we have $\hat{J}_n \leq C_N \, n^{\frac{1}{2s + 1}}$, and we may take $\hat{J}_n \asymp n^{\frac{1}{2s + 1}}$. Re-indexing $\hat{f}_n$, we write
    \begin{equation*}
        \hat{f}_n = \textstyle\sum_{l = 1}^{\hat{J}_n} \hat{\alpha}_l B_k(\cdot ; \hat\bxi_{k,l}).
    \end{equation*}
	Here, $B_k(\cdot ; \hat{\bxi}_{k,l} )$ are the cardinal B-splines appearing in \textbf{Lemma 2} of \citet{suzuki2018adaptivity}, and $\hat{\bxi}_{k,l}$ are the uniform knots defining each basis. Moreover, from the proof of \textbf{Proposition 1} in \citet{suzuki2018adaptivity}, an upper bound on the coefficients $\hat\balpha = (\hat\alpha_l)_l$ is obtained,
	\begin{equation*}
		\| \hat\balpha \|_\infty \lesssim n^{ \frac{1}{2s + 1} (\nu^{-1} + 1) \left( \frac{1}{p} - s \right)_+}.
	\end{equation*}
	Thus, for some constant $\hat{C}_B > 0$, we may define $\hat{B}_n := \hat{C}_B \, n^{\frac{1}{2s + 1} (\nu^{-1} + 1) \left( \frac{1}{p} - s \right)_+}$ so that $\| \hat\balpha \|_\infty \leq \hat{B}_n$. Next, consider the spacing between the knots of the cardinal B-splines in $\hat{f}_n$. From \textbf{Lemma 2}, the minimal spacing between knots is
	\begin{equation*}
		2^{- \lceil \nu^{-1} \log{(\lambda N)} \rceil - \lceil C_1 \log{N} \rceil - 1},
	\end{equation*}
	where $C_1$ and $\lambda > 0$ are constants independent of $n$. Hence,
	\begin{align*}
		2^{- \lceil \nu^{-1} \log{(\lambda N)} \rceil - \lceil C_1 \log{N} \rceil - 1} &\geq 2^{ - \left( \frac{C_1 + \nu^{-1}}{2s + 1} \log{n} \right) - 3 - \nu^{-1} \log{\lambda}} =: \hat{\delta}_n.
	\end{align*}
	With this definition of $\hat{\delta}_n$, the approximation $\hat{f}_n$ belongs to the model space
	\begin{equation*}
		\hat\calF_n := \left\{ f = \sum_{l=1}^{\hat{J}_n} \beta_l B(\cdot ; \bxi_{k,l}) : \| \bbeta \|_\infty \leq \hat{B}_n, \; m(\bxi_{k,l}) \geq \hat{\delta}_n, \; \forall l \right\}.
	\end{equation*}
	This completes the proof.
\end{proof}

\section{Proof of Theorem 4.1.}
\label{sec:main_thm_proof}

\begin{lemma}	\label{eq:covering}
	Let $J \in \bbN$, $B > 0$, and $0 < \delta \leq 2 (1 + 2A) / \overline{k}$. For all $0 < \epsilon \leq B$, the following holds:
	\begin{equation*}
		\calN (\epsilon, \calF(J, B, \delta), \|\cdot\|_{L^2(P_X)}) \leq (J + |S|)^{|S|} \left( \frac{6 R J B }{\epsilon} \right)^{J} \left( \frac{3 R L (B, J, \delta; \overline{k}) (1 + 2A)}{\epsilon} \right)^{J (\overline{k} + 2)}
	\end{equation*}
	where $L(B, J, \delta; \overline{k}) := 4 |S| (1 + 2A)^{\overline{k}} (\overline{k} + 2) B J (\delta \wedge 1)^{- (\overline{k} + 1) }$, and $\overline{k} = \textstyle\max_{k \in S} k.$
\end{lemma}

\begin{proof}
	By assumption (A1), we have $\| f \|_{L^2 (P_X)} \leq R \| f \|_{L^\infty}$. Moreover, by the definition of \eqref{eq:Sieve}, it is easy to see that
	\begin{align*}
		\calN (\epsilon, \calF(J, B, \delta), \| \cdot \|_{L^2(P_X)} ) &\leq \calN \left(\frac{\epsilon}{R}, \calF (J, B, \delta), \|\cdot\|_{L^\infty} \right)   \\
        &\leq \sum_{\boldsymbol{J}: \sum_{k \in S} J_k \leq J} \calN \left( \frac{\epsilon}{R}, \calF (\boldsymbol{J}, B, \delta), \|\cdot\|_{L^\infty} \right).	
	\end{align*}
	We now derive an upper bound for the covering number inside the sum. Let $\{ \bbeta_1, \ldots, \bbeta_N \}$ and $\{ \bxi_1, \ldots, \bxi_M \}$ be a $\eta_1$-net of $\calC( \bJ, B)$ and a $\eta_2$-net of $\calX_\xi (\bJ, \delta)$, respectively. Given any $f_{\bbeta, \bxi} \in \calF (\bJ, B, \delta)$, there exist elements $\bbeta_i$ and $\bxi_j$ in the cover such that $f_{\bbeta_i, \bxi_j} \in \calF(\bJ, B, \delta)$, and by \textbf{Lemma~\ref{eq:complexity}},
	\begin{align*}
		\|f_{\bbeta,\bxi} - f_{\bbeta_i,\bxi_j}\|_{L^\infty} &\leq \|f_{\bbeta,\bxi} - f_{\bbeta_i,\bxi_j}\|_{L^\infty(\calX)}	\\
		&\leq \|f_{\bbeta,\bxi} - f_{\bbeta_i, \bxi}  \|_{L^\infty(\calX)} + \| f_{\bbeta_i,\bxi} - f_{\bbeta_i,\bxi_j} \|_{L^\infty(\calX)}	\\
		&\leq \left( \sum_{k \in S} J_k \right) \| \bbeta -\bbeta_i \|_\infty + L(B, \bJ, \delta ; \overline{k} ) \cdot \| \bxi - \bxi_j \|_\infty,
	\end{align*}
	where $L(B, \bJ, \delta ; \overline{k}) = 4 |S| (1 + 2A)^{\overline{k}} (\overline{k} + 2) B (\sum_{k \in S} J_k) (\delta \wedge 1)^{- ( \overline{k} + 1 )}$ is the Lipschitz bound appearing in \textbf{Lemma~\ref{eq:complexity}}. Therefore, by taking $\eta_1 = \frac{\epsilon}{2 R \sum_{k \in S} J_k}$ and $\eta_2 = \frac{\epsilon}{2 L(B, \bJ, \delta ; \overline{k}) R}$, the $\frac{\epsilon}{R}$-covering number of $\calF (\bJ, B, \delta)$ is always bounded above by $N M$. Since $\calC (\bJ, B) = [-B, B]^{\sum_{k \in S} J_k}$, $\calX_\xi (\bJ, \delta) \subset \calX^{\sum_{k \in S} J_k (k + 2)}$, and $\calX = [-A, 1 + A]$, it follows that
    \begin{align*}
		&\calN \left( \frac{\epsilon}{R}, \calF(\boldsymbol{J}, B, \delta), \|\cdot\|_{L^\infty} \right)  \\
        &\leq \calN \left( \frac{\epsilon}{2R (\sum_{k \in S} J_k)}, [-B, B]^{\sum_{k \in S} J_k}, \| \cdot \|_\infty \right)   \\
        &\quad \times  \calN \left( \frac{\epsilon}{2 L(B, \bJ, \delta) R}, [-A, 1 + A]^{\sum_{k \in S} J_k (k + 2)}, \|  \cdot \|_\infty \right) \\
        &\leq \left( \frac{6 R \sum_{k \in S} J_k B }{\epsilon} \right)^{\sum_{k \in S} J _k} \times \left( \frac{3 R L(B, \bJ, \delta; \overline{k}) (1 + 2A)}{\epsilon} \right)^{\sum_{k \in S} J_k (k + 2)},
    \end{align*}
	for all $0 < \epsilon \leq B$.
	Furthermore, the number of tuples $\boldsymbol{J} = (J_k)_{k \in S}$ satisfying $\textstyle \sum_{k \in S} J_k \leq J$ is bounded by $\binom{J + |S|}{|S|} \leq (J + |S|)^{|S|}$. Additionally, note that
    \begin{equation*}
        L(B, \bJ, \delta; \overline{k} ) \leq 4 |S| (1 + 2A)^{\overline{k}} (\overline{k} + 2) B J (\delta \wedge 1)^{- ( \overline{k} + 1 )}  =: L(B, J, \delta; \overline{k} )
    \end{equation*}
    uniformly over $\bJ : \sum_{k \in S} J_k \leq J$. Combining these bounds yields
    \begin{align*}
        \calN (\epsilon, \calF(J, B, \delta), \|\cdot\|_{L^2(P_X)}) &\leq \sum_{\bJ : \sum_{k \in S} J_k \leq J} \calN \left( \frac{\epsilon}{R}, \calF (\bJ, B, \delta), \| \cdot \|_{L^\infty} \right)    \\
        &\leq (J + |S|)^{|S|} \left( \frac{6 R J B }{\epsilon} \right)^{J} \left( \frac{3 R L (B, J, \delta; \overline{k}) (1 + 2A)}{\epsilon} \right)^{J (\overline{k} + 2)}
    \end{align*}
\end{proof}

\begin{proposition}	\label{eq:entropy}
	For the sieve (\ref{eq:Sieve}), the following holds:
	\begin{equation*}
		\sup_{\epsilon > \epsilon_n} \log \calN \left( \frac{\epsilon}{36}, \Theta_n, d_{n,H} \right) \lesssim n\epsilon_n^2.
	\end{equation*}
\end{proposition}

\begin{proof}
	First, by \textbf{Lemma 2.7} of \citet{ghosal2017fundamentals}, for $(f_1, \sigma_1), (f_2, \sigma_2) \in \Theta_n$,
	\begin{equation*}
		d_{n, H} (p_{\theta_1}, p_{\theta_2}) \leq \frac{\|f_1 - f_2\|_{L^2(P_X)}}{\sigma_1 + \sigma_2} + \frac{2|\sigma_1 - \sigma_2|}{\sigma_1 + \sigma_2} \leq \frac{\|f_1 - f_2\|_{L^2(P_X)}}{2 \underline{\sigma}} + \frac{|\sigma_1 - \sigma_2|}{\underline{\sigma}} 
	\end{equation*}
	holds. Therefore,
	\begin{align}  \label{eq:entropy1}
    \begin{split}
		\calN \left( \frac{\epsilon}{36}, \Theta_n, d_{n, H} \right) &\leq \calN \left( \frac{\underline{\sigma}_n \epsilon}{72}, \calF_n, \|\cdot\|_{L^2(P_X)} \right) \cdot \calN \left( \frac{\underline{\sigma}_n \epsilon}{36}, [\underline{\sigma}_n, \overline{\sigma}_n], |\cdot| \right)  \\
        &\lesssim \calN \left( \frac{\underline{\sigma}_n \epsilon}{72}, \calF_n, \|\cdot\|_{L^2(P_X)} \right) \cdot \frac{\overline\sigma_n - \underline\sigma_n}{\underline\sigma_n \epsilon}.
    \end{split}
	\end{align}
	Next, assume that $\epsilon > \epsilon_n$. We apply \textbf{Lemma~\ref{eq:covering}} to the first term. The second term can be bounded easily. Ignoring constant factors, an upper bound for \eqref{eq:entropy1} is given by
	\begin{equation}   \label{eq:covering2}
		(J_n + |S|)^{|S|} \left[ \frac{432 R J_n B_n}{\underline{\sigma}_n \epsilon_n} \right]^{J_n} \left[ \frac{216 R L_n (1 + 2A)}{\underline{\sigma}_n \epsilon_n} \right]^{\overline{k} J_n}  \frac{\overline\sigma_n - \underline\sigma_n}{\underline\sigma_n \epsilon_n}.
	\end{equation}
	Here, the constant $L_n$ is from \textbf{Lemma~\ref{eq:covering}} and
    \begin{equation*}
        L_n := L(B_n, J_n, \delta_n ; \overline{k} ) = 4 |S| (1 + 2A)^{\overline{k}} (\overline{k} + 2) B_n J_n (\delta_n \wedge 1)^{- ( \overline{k} + 1 )}.
    \end{equation*}
    Note that $L_n \asymp B_n J_n (\delta_n \wedge 1)^{-(\overline{k} + 1)}$. Taking logs of \eqref{eq:covering2} and expanding while ignoring constant terms, we obtain
	\begin{equation}   \label{eq:covering3}
		\log J_n + J_n \log \frac{J_n B_n}{\underline{\sigma}_n \epsilon_n} + \log ( \overline\sigma_n - \underline{\sigma}_n ) - \log \underline\sigma_n - \log \epsilon_n.
	\end{equation}
	Finally, substituting the values in \eqref{eq:SieveCond} into \eqref{eq:covering3}, we can bound the quantity by a constant multiple of $n \epsilon_n^2 = n^{\frac{1}{2s + 1}} ( \log n )^4$.
\end{proof}

\begin{proposition}  \label{eq:tail}
	For some constant $C > 2$, the following holds:
\begin{equation*}
	\Pi(\Theta \setminus \Theta_n) = o(e^{-(C + 2) n \epsilon_n^2}).
\end{equation*}
\end{proposition}

\begin{proof}
	By the definition \eqref{eq:Sieve}, it is immediate that the following inequality holds:
\begin{equation}   \label{eq:tail1}
    \Pi (\Theta \setminus \Theta_n) \leq  \Pi( \sigma \notin [\underline\sigma_n, \overline\sigma_n] ) + \Pi (f \notin \mathcal{F}_n) =: I_1 + I_2.
\end{equation}
	Therefore, it suffices to show that $(I_1 + I_2) \cdot e^{(C + 2) n \epsilon_n^2} = o (1)$, which completes the proof. First, consider $I_1$. Since the prior is given by $\sigma^{-2} \sim \text{Gamma}(\alpha, \lambda)$ with hyperparameters $\alpha = r/2$ and $\lambda = rR / 2$, we decompose $I_1$ as
\begin{equation}   \label{eq:tail2}
    I_1 = \Pi (\sigma^{-2} < \overline{\sigma}_n^{-2})  + \Pi (\sigma^{-2} > \underline{\sigma}_n^{-2}).
\end{equation}
	For the first term in \eqref{eq:tail2}, since $e^{-\lambda t} \le 1$ for all $t \ge 0$, we have
\begin{equation}   \label{eq:tail3}
    \Pi (\sigma^{-2} < \overline{\sigma}_n^{-2}) = \frac{\lambda^\alpha}{\Gamma(\alpha)} \int_0^{\overline{\sigma}_n^{-2}} t^{\alpha - 1} e^{-\lambda t} dt \leq \frac{\lambda^\alpha}{\Gamma(\alpha)} \int_0^{\overline{\sigma}_n^{-2}} t^{\alpha - 1} dt.
\end{equation}
	Integrating the right-hand side shows that it is proportional, up to a constant factor, to $(\overline{\sigma}_n)^{-2\alpha}$. For the second term, we apply the moment generating function (MGF) of the Gamma distribution together with Markov's inequality. For any $0 < t < \lambda$,
\begin{equation}   \label{eq:tail4}
    \Pi (\sigma^{-2} > \underline{\sigma}_n^{-2}) \leq e^{-t \underline{\sigma}_n^{-2}} \left( 1 - \frac{t}{\lambda} \right)^{-\alpha}.
\end{equation}
	Substituting $t = \lambda / 2$, the right-hand side becomes proportional to $2^\alpha e^{- \frac{\lambda}{2} \underline{\sigma}_n^{-2} }$. Combining the bounds from \eqref{eq:tail3} and \eqref{eq:tail4}, it is enough to show that
\begin{equation}    \label{eq:tail5}
    \left( e^{-2 \alpha \log \overline{\sigma}_n}  + e^{- \frac{\lambda}{2} \underline{\sigma}_n^{-2}  } \right) \cdot e^{(C + 2) n \epsilon_n^2} = o(1).
\end{equation}
	Now plug in $\underline{\sigma}_n = C_{\underline\sigma} \, n^{- \frac{1}{2(2s + 1)}} (\log n )^{-2}$ and $\overline{\sigma}_n = e^{C_{\overline\sigma} \, n^{\frac{1}{2s + 1}} (\log n )^{4}}$ defined in \eqref{eq:SieveCond}, together with $n \epsilon_n^2 = n^{\frac{1}{2s + 1}} (\log n)^4$. Then \eqref{eq:tail5} holds for sufficiently large constants $C_{\overline{\sigma}} > (C + 2) / r$ and $C_{\underline{\sigma}} < \sqrt{\frac{\lambda}{2 (C + 2)}}$.

	Next, we handle the term $I_2$ in \eqref{eq:tail1}.
	\begin{align*}
		\Pi(f \notin \calF_n) &\leq \Pi \left( \boldsymbol{J} : \sum_{k \in S} J_k > J_n \right) + \Pi \left( (\boldsymbol{J}, \boldsymbol{\beta}) : \sum_{k \in S} J_k \leq J_n, \; \| \bbeta \|_\infty > B_n \right)	\\
		&\; + \Pi \left( (\boldsymbol{J}, \bxi) : \sum_{k \in S} J_k \leq J_n , \; m(\bxi_{k,l}) < \delta_n \text{ for some } \bxi_{k,l} \right)	\\
		&=: I_{21} + I_{22} + I_{23}.
	\end{align*}
	First, under the truncated uniform prior specification (P5), we have $I_{23} = 0$. Hence it remains to show $(I_{21} + I_{22} ) \cdot e^{(C + 2)n \epsilon_n^2} = o(1)$.
	
	For $I_{21}$, since the specified priors implies the negative binomial marginal, we have $J_k \stackrel{iid}{\sim} nb(a_{k}, b_{n} / (1 + b_{n})), \; \forall k$. The MGF of the negative binomial distribution and Markov's inequality imply that
	\begin{equation}   \label{eq:tail6}
		I_{21} \leq e^{-t J_n} \prod_{k \in S} \left( \frac{b_{n}}{1 + b_{n} - e^t} \right)^{a_{k}}, \quad 0 \leq t < \log(1 + b_n).
	\end{equation}
	From the gamma prior specification (P3), we have $b_n = e^{C_b (\log n)^2}$ with $C_b > 0$. Take $t = t_n := \log{b_n}$ and substitute it into \eqref{eq:tail6}. Then it is sufficient to show that
	\begin{equation}   \label{eq:tail7}
		\exp{ \left( -t_n J_n + \big( \sum_{k \in S} a_{k} \big) \log b_n + (C + 2) n\epsilon_n^2 \right) } = o(1).
	\end{equation}
	Moreover, since \eqref{eq:SieveCond} gives $J_n = \lceil C_J \, n^{\frac{1}{2s + 1}}  (\log n)^2 \rceil$, \eqref{eq:tail7} holds whenever we choose $C_J > \frac{(C + 2)}{C_b}$. Next, consider $I_{22}$.
	\begin{align}  \label{eq:tail8}
    \begin{split}
		I_{22} &:= \sum_{\sum_{k \in S} J_k \leq J_n} \Pi (\boldsymbol{J}) \cdot \Pi \big( \bbeta : \| \bbeta \|_\infty > B_n \, | \, \boldsymbol{J} \big)	\\
		&\leq \sup_{\sum_{k \in S} J_k \leq J_n} \Pi (\bbeta : \| \bbeta \|_\infty > B_n \, | \, \boldsymbol{J}).
    \end{split}
	\end{align}
	Here $\pi(\boldsymbol{J}) = \prod_{k \in S} \pi_{k,n} (J_k)$ is the joint pmf of independent negative binomial distributions, where $\pi_{k,n}$ denotes the pmf of $nb(a_{k}, b_{n} / (1 + b_{n}))$. Since (P4) assumes $\beta_{k,l} \sim N(0, \phi_n^2)$ with $\phi_n = e^{C_\phi (\log n)^2}$ and $C_\phi > 0$, it follows that
	\begin{equation}   \label{eq:tail9}
		\sup_{\sum_{k \in S} J_k \leq J_n} \Pi \big( \bbeta : \| \bbeta \|_\infty > B_n \, | \, \boldsymbol{J} \big) \lesssim J_n \cdot \exp{ \left( - \frac{B_n^2}{2 \phi_n^2} \right) }.
	\end{equation}
	Combining \eqref{eq:tail8} and \eqref{eq:tail9}, the proof is completed by showing
	\begin{equation}   \label{eq:tail10}
		\exp{ \left( \log J_n - \frac{1}{2} \frac{B_n^2}{\phi_n^2} + (C + 2) n \epsilon_n^2 \right)} = o(1).
	\end{equation}
	Substituting $B_n := C_B \, \phi_n \, n^{\frac{1}{2(2s + 1)}} ( \log n)^2$ with $C_B > 0$ from \eqref{eq:Sieve}, we obtain that \eqref{eq:tail10} holds whenever $C_B > \sqrt{2 (C + 2)}$. This completes the proof.
\end{proof}

\begin{theorem}
	Consider the model assumptions (A1) and the prior setting (P1)--(P5). Suppose that $f_0 \in B^{s}_{p, q} ([0, 1])$ satisfies $s > \left(\frac{1}{p} - \frac{1}{2} \right)_+$ and the model contains the degree $k$ B-splines satisfying $s < \min \{k, k - 1 + \frac{1}{p} \}$. Then there exists a constant $C > 2$ such that
\begin{equation*}
	\Pi \left( B_{n,2}^\ast (\theta_0, \varepsilon_n) \right) \geq e^{-C n \varepsilon_n^2},
\end{equation*}
	for all sufficiently large $n$.
\end{theorem}

\begin{proof}
	By \textbf{Lemma 2.7} of \citet{ghosal2017fundamentals}, we have
	\begin{align}  \label{eq:KL1}
    \begin{split}
		K(p_{f_0, \sigma_0, i}, p_{f, \sigma, i}) &= \frac{1}{2} \|f_0 - f\|^2_{L^2(P_X)} + \frac{1}{2} \big[ \log{ ( \sigma^2 / \sigma_0^2 )} + ( \sigma_0^2 / \sigma^2 ) - 1 \big] \\
		V_{2, 0} \left( p_{f_0, \sigma_0, i}, p_{f, \sigma, i} \right) &= \left( \frac{\sigma_0}{\sigma^2} \right)^2 \|f_0 - f\|_{L^2(P_X)}^2 + \frac{(\sigma_0^2 - \sigma^2)^2}{2\sigma^4}.
    \end{split}
	\end{align}
	From this, define the new set
	\begin{equation*}
		\hat{B}_{n,2} (\theta_0, \epsilon) := \left\{(f, \sigma) : \|f - f_0\|_{L^2(P_X)} \leq \frac{\sigma_0 \epsilon}{2}, \; \sigma_0^2 \leq \sigma^2 \leq (1 + \epsilon^2) \sigma_0^2 \right\}.
	\end{equation*}
	Combining this with \eqref{eq:KL1}, we obtain
	\begin{align*}
		\Pi (B^\ast_{n,2}(\theta_0, \epsilon_n)) &\geq \Pi (\hat{B}_{n,2}(\theta_0, \epsilon_n))	\\
		&= \Pi \left( \sigma^2 \in [\sigma_0^2, (1 + \epsilon_n^2) \sigma_0^2] \right) \cdot \Pi \left( f: \|f - f_0 \|_{L^2(P_X)} \leq \frac{\sigma_0 \epsilon_n}{2} \right)	\\
		&=: P_1 \cdot P_2.
	\end{align*}
	It now suffices to show that there exists a constant $C > 0$ such that, for all sufficiently large $n$, $- \log ( P_1 \cdot P_2 ) \lesssim n \epsilon_n^2$. We first bound $P_1$. Since $\sigma^2 \sim Inv\text{-}Gam(\alpha, \lambda)$ with $\alpha = r / 2$ and $\lambda = rR / 2$, applying the mean value theorem yields that there exists a constant $\tilde{c}_1 > 0$ such that
	\begin{equation*}
		P_1 = \int_{\sigma_0^2}^{(1 + \epsilon_n^2)\sigma_0^2} \frac{\lambda^\alpha}{\Gamma(\alpha)} x^{- \alpha - 1} e^{- \lambda / x} dx \geq \tilde{c}_1 \epsilon_n^2.
	\end{equation*}
	Therefore,
    \begin{equation*}
        -\log{P_1} \lesssim \log{ ( 1 / \epsilon_n )} \lesssim n \epsilon_n^2.
    \end{equation*}

	Next, consider $P_2$. Since $\|f - g\|_{L^2(P_X)} \leq R \|f - g\|_{L^2}$, we have
	\begin{equation}   \label{eq:KL2}
		P_2 \geq \Pi \left(f: \|f - f_0\|_{L^2} \leq \frac{\sigma_0 \epsilon_n}{2R} \right).
	\end{equation}
	Moreover, by \textbf{Lemma~\ref{eq:apprxlem}}, for $k^\star \in S$ satisfying $s < \min \{ k^\star, k^\star + 1/p - 1\}$, define
    \begin{align*}
        \hat{\calF}_n := \left\{ f = \sum_{l=1}^{\hat{J}_n} \alpha_l B( \cdot; \bxi_{k^\star,l}) : \| \balpha\|_\infty \leq \hat{B}_n, \; m(\bxi_{k^\star,l}) \geq \hat{\delta}_n \text{ for all } l \right\}
    \end{align*}
    with $\hat{J}_n \asymp n^{\frac{1}{2s + 1}} $, $\hat{B}_n = n^{\frac{1}{2s + 1} (\nu^{-1} + 1) \left( \frac{1}{p} - s \right)_+}$, $\nu = (s - \omega) / (2 \omega)$, $\omega = (1 / p - 1 / 2)_+$, and $\log ( 1 / \hat\delta_n) \asymp \log n$.
    Then there exists a function $\hat{f}_n = \textstyle\sum_{l=1}^{\hat{J}_n} \hat{\alpha}_l B( \cdot ; \hat{\bxi}_{k^\star,l}) \in \hat{\calF}_n$ consisting of B-splines of degree $k^\star$ with parameters $\hat{\balpha} = (\hat\alpha_l)_l$ and $\hat{\bxi} = (\hat\bxi_{k^\star, l})_{l}$, such that
	\begin{equation*}
		\| \hat{f}_n - f_0 \|_{L^2} \leq \tilde{c}_2 \, n^{-\frac{s}{2s + 1}}
	\end{equation*}
	for some $\tilde{c}_2 > 0$. The right-hand side can be bounded by a constant multiple of $\epsilon_n = n^{- \frac{s}{2s + 1}} (\log n)^2$. Since $\| f - f_0 \|_{L^2} \leq \| f - \hat{f}_n \|_{L^2} + \| \hat{f}_n - f_0 \|_{L^2}$, $\| f - \hat{f}_n \|_{L^2} \leq \| f - \hat{f}_n \|_{L^2(\calX)}$, and $[0, 1] \subset \calX$, a lower bound for \eqref{eq:KL2} is given by the following: there exists a constant $\tilde{c}_3 > 0$ such that
	\begin{equation*}
		\Pi \left(f: \|f - f_0\|_{L^2} \leq \frac{\sigma_0 \epsilon_n}{2R} \right) \geq \Pi (f \in \hat{\calF}_n: \| f - \hat{f}_n \|_{L^2 (\calX)} \leq \tilde{c}_3 \, \epsilon_n).
	\end{equation*}
	Applying \textbf{Lemma~\ref{eq:complexity}} to the lower bound on the right-hand side yields
    \begin{equation*}
        \| f - \hat{f}_n \|_{L^2 (\calX)} \leq R \| f - \hat{f}_n \|_{L^\infty} \leq R \hat{J}_n \| \bbeta - \hat{\balpha} \|_\infty + \hat{L}_n \| \bxi - \hat{\bxi}\|_\infty
    \end{equation*}
    where $\hat{L}_n := L(\hat{B}_n, \hat{J}_n, \hat{\delta}_n ; \overline{k}) = 4 (1 + 2A)^{\overline{k}} (\overline{k} + 2) \hat{B}_n \hat{J}_n (\hat{\delta}_n \wedge 1)^{- (\overline{k} + 1) }$.
    Thus, the quantity has a lower bound
    \begin{align}   \label{eq:KL3}
    \begin{split}
        \pi_{k,n} (\hat{J}_n) \, &\Pi \left( \bbeta \in \calC (\hat{J}_n, \hat{B}_n) : \| \bbeta - \hat{\balpha} \|_\infty \leq \frac{\tilde{c}_4 \epsilon_n}{\hat{J}_n}, \, \big| \, \hat{J}_n \right)    \\
        &\cdot \Pi \left(\bxi \in \calX_\xi (\hat{J}_n, \hat\delta_n) : \| \bxi - \hat{\bxi}_n \|_\infty \leq \frac{\tilde{c}_5 \epsilon_n}{\hat{L}_n} \, \big| \, \hat{J}_n \right)  \\
        &=: P_{21} P_{22} P_{23}.
    \end{split}
    \end{align}
	for some $\tilde{c}_4, \tilde{c}_5 > 0$ and $\calC (\hat{J}_n, \hat{B}_n) := [-\hat{B}_n, \hat{B}_n]^{\hat{J}_n}$, $\calX_{\xi} (\hat{J}_n, \hat{\delta}_n) := \{ \bxi = (\bxi_{k^\star, l})_l \in \calX^{(k^\star + 2) \hat{J}_n} : m( \bxi_{k^\star, l } ) \geq \hat{\delta}_n, \, l = 1, \ldots, \hat{J}_n \}$. Here, $\pi_{k^\star,n}$ denotes the density of $J_{k^\star} \sim nb(a_{k^\star}, b_n / (1 + b_n) )$. Also, note that $\hat{L}_n \asymp \hat{B}_n \hat{J}_n (\hat{\delta}_n \wedge 1 )^{- (\overline{k} + 1)}$. It remains to show that $- (\log P_{21} + \log P_{22} + \log P_{23}) \lesssim n\epsilon_n^2$.

	We start with $P_{21}$. Substituting $\hat{J}_n \asymp n^{\frac{1}{2s + 1}}$ and $b_n = e^{C_b (\log n)^2}$ into the density $\pi_{k^{\star},n}$ and applying Stirling's approximation, we obtain that there exists a constant $\tilde{c}_6 > 0$ such that
	\begin{equation*}
		P_{21} \geq \tilde{c}_6 \, b_n^{-(\hat{J}_n + a_{k^\star})}.
	\end{equation*}
	Therefore, $- \log P_{21} \lesssim n^{\frac{1}{2s + 1}} (\log n)^2 \lesssim n\epsilon_n^2$, which bounds $P_{21}$.

	Next, consider $P_{22}$ in \eqref{eq:KL3}. Before proceeding, let $\hat\kappa_n = \frac{\tilde{c}_4 \epsilon_n}{\hat{J}_n}$. Each component of $\bbeta$ satisfies $\beta_{k^\star,l} \stackrel{iid}{\sim} N(0, \phi_n^2)$ for $l = 1, \ldots, \hat{J}_n$. Let $\varphi$ denote the density of $N(0,1)$. Then
	\begin{align*}
		P_{22} &\geq \prod_{l=1}^{\hat{J}_n} \int_{( \hat{\alpha}_l - \hat{\kappa}_n) \vee (- \hat{B}_n )}^{(\hat{\alpha}_l + \hat{\kappa}_n ) \wedge \hat{B}_n} \frac{1}{\phi_n} \varphi \left( \frac{x}{\phi_n} \right) dx	\\
		&\gtrsim \left[ \frac{\hat\kappa_n}{\phi_n} e^{-\frac{\hat{B}_n^2}{2\phi_n^2}} \right]^{\hat{J}_n}.
	\end{align*}
	Taking $-\log$ and ignoring constant terms yields
	\begin{equation*}
		-\log P_{22} \lesssim \hat{J}_n \log \frac{\phi_n}{\hat\kappa_n} + \frac{\hat{J}_n \hat{B}_n^2}{\phi_n^2} \lesssim n \epsilon_n^2.
	\end{equation*}
	The second inequality follows by substituting $\hat{J}_n \asymp n^{\frac{1}{2s + 1}}$, $\hat{\kappa}_n \asymp n^{-\frac{s + 1}{2s + 1}} (\log n)^2$, $\phi_n = e^{C_\phi (\log n)^2}$, $\hat{B}_n \asymp n^{\frac{1}{2s + 1} (\nu^{-1} + 1) \left( \frac{1}{p} - \frac{1}{2} \right)_+}$, and $\nu = (s - \omega) / (2 \omega)$ where $\omega = (1 / p - 1 / 2)_+$. 

	Finally, we bound $P_{23}$ in \eqref{eq:KL3}. We have $\bxi_{k^\star,l} = (\xi_{k^\star,l,j})_{j=1}^{k^\star+2} \stackrel{iid}{\sim} U(\calX^{k^\star + 2}(\hat\delta_n))$ for $l = 1, \ldots, \hat{J}_n$, where
	$\calX^{k + 2} (\hat\delta_n) = \{ \bfz \in \calX^{k + 2} : z_1 < \cdots < z_{k + 2}, \; \min_j (z_{j + 1} - z_j) \geq \hat{\delta}_n \}$.
	From \eqref{eq:approximation_space} in \textbf{Lemma~\ref{eq:apprxlem}}, we can verify that $\hat\delta_n \asymp n^{-\tilde{c}_7}$ for some $\tilde{c}_7 > 0$. Let $\hat{\eta}_n = \frac{\tilde{c}_5 \epsilon_n}{\hat{L}_n}$. Since each $\bxi_{k,l}$ is independent, we have
    \begin{equation}    \label{eq:KL4}
        P_{23} \geq \prod_{l=1}^{\hat{J}_n} \Pi \big(\bxi_{k^\star,l} \in \calX^{k^\star + 2}(\hat\delta_n) : \| \bxi_{k^\star,l} - \hat\bxi_{k^\star,l} \|_\infty \leq \hat{\eta}_n \big).
    \end{equation}
	Consider the set
    \begin{equation*}
        B (\hat{\bxi}_{k^\star,l}, \hat\eta_n) = \{\bxi_{k^\star,l} \in \calX^{k^\star + 2} : \| \bxi_{k^\star,l} - \hat{\bxi}_{k^\star,l} ||_\infty \leq \hat\eta_n \},
    \end{equation*}
	which is the $\hat\eta_n$-ball around the uniform knot vector $\hat\bxi_{k,l}$ that defines the approximator $\hat{f}_n$. Note that
    \begin{align*}
        | \xi_{k^\star,l, j + 1} - \xi_{k^\star,l, j} | &\geq | \hat\xi_{k^\star,l,j+1} - \hat\xi_{k^\star,l,j}| - |\xi_{k^\star,l, j + 1} - \hat\xi_{k^\star,l,j+1} - (\xi_{k^\star,l,j} - \hat\xi_{k^\star,l,j} ) |    \\
        &\geq \hat\delta_n - 2 \hat\eta_n,
    \end{align*}
	and since $\epsilon_n$, $\hat{\delta}_n \rightarrow 0$ and $\hat{B}_n$, $\hat{J}_n \rightarrow \infty$ as $n \rightarrow \infty$, we have $\hat\eta_n \asymp \frac{\epsilon_n ( \hat\delta_n \wedge 1 )^{\overline{k} + 1}}{ \hat{B}_n \hat{J}_n} \lesssim \hat{\delta}_n$ for large $n$. Thus, $B (\hat\bxi_{k^\star,l}, \hat\eta_n) \subset \calX^{k^\star + 2} (\hat\delta_n)$ eventually. Hence, for sufficiently large $n$, the right-hand side of \eqref{eq:KL4} is always positive. Moreover, each probability is proportional to the volume of $B (\hat\bxi_{k^\star,l}, \hat\eta_n)$, yielding the lower bound
   \begin{align*}
       P_{23} \gtrsim \hat\eta_n^{(k^\star + 2) \hat{J}_n}.
   \end{align*}
	Taking $-\log$ and using $\hat{J}_n \asymp n^{\frac{1}{2s + 1}}$ and $-\log \hat\eta_n = - \log \frac{\tilde{c}_5 \epsilon_n}{\hat{L}_n} \lesssim \log{n}$, we obtain
    \begin{align*}
        -\log P_{23} \lesssim n^{\frac{1}{2s + 1}} \log n \lesssim n \epsilon_n^2,
    \end{align*}
	which completes the proof.
\end{proof}

\end{document}